\documentclass[twoside,11pt]{article}

\usepackage[abbrvbib,preprint]{jmlr2e}

\usepackage{color}
\usepackage{enumerate}
\usepackage{bm}
\usepackage{amsmath}
\usepackage{booktabs}
\usepackage[ruled,linesnumbered]{algorithm2e}
\usepackage{multirow}
\usepackage{enumitem}
\usepackage{marvosym}

\newtheorem{mytheorem}{Theorem}
\newtheorem{myassumption}{Assumption}
\newtheorem{mylemma}{Lemma}

\newcommand{\eqbf}[1]{\boldsymbol{#1}}
\newcommand{\ie}{\emph{i.e.}}
\newcommand{\eg}{\emph{e.g.}}
\ShortHeadings{Adaptive Sampled Softmax with Inverted Multi-Index}{Jin Chen}
\firstpageno{1}

\begin{document}

\title{Adaptive Sampled Softmax with Inverted Multi-Index: Methods, Theory and Applications}

\author{\name Jin Chen \email jinchen@ust.hk \\
       \addr Hong Kong University of Science and Technology \\
       Hong Kong
       \AND
       \name Jin Zhang \email jinzhang21@mail.ustc.edu.cn \\
       % \addr University of Science and Technology of China\\
       % \& State Key Laboratory of Cognitive Intelligence \\
       % Hefei, China
       % \AND
       \name Xu Huang \email xuhuangcs@mail.ustc.edu.cn \\
       \addr University of Science and Technology of China\\
       \& State Key Laboratory of Cognitive Intelligence \\
       Hefei, China
       \AND
       \name Yi Yang \email imyiyang@ust.hk \\
       \addr Hong Kong University of Science and Technology \\
       Hong Kong
       \AND
       \name Defu Lian \Letter \email liandefu@ustc.edu.cn\\
       % \addr University of Science and Technology of China\\
       % \& State Key Laboratory of Cognitive Intelligence \\
       % Hefei, China
       % \AND
       \name Enhong Chen \email cheneh@ustc.edu.cn \\
       \addr University of Science and Technology of China\\
       \& State Key Laboratory of Cognitive Intelligence  \\
       Hefei, China
       }

% \editor{Kevin Murphy and Bernhard Sch{\"o}lkopf}

\maketitle

\begin{abstract}%   <- trailing '%' for backward 
The softmax function is a cornerstone of multi-class classification, integral to a wide range of machine learning applications, from large-scale retrieval and ranking models to advanced large language models. However, its computational cost grows linearly with the number of classes, which becomes prohibitively expensive in scenarios with millions or even billions of classes. 
The sampled softmax, which relies on self-normalized importance sampling, has emerged as a powerful alternative, significantly reducing computational complexity. Yet, its estimator remains unbiased only when the sampling distribution matches the true softmax distribution.
To improve both approximation accuracy and sampling efficiency, we propose the \texttt{MIDX} Sampler,  a novel adaptive sampling strategy based on an inverted multi-index approach.
Concretely, we decompose the softmax probability into several multinomial probabilities, each associated with a specific set of codewords and the last associated with the residual score of queries, thus reducing time complexity to the number of codewords instead of the number of classes. To further boost efficiency, we replace the query-specific residual probability with a simple uniform distribution, simplifying the computation while retaining high performance. Our method is backed by rigorous theoretical analysis, addressing key concerns such as sampling bias, gradient bias, convergence rates, and generalization error bounds. The results demonstrate that a smaller divergence from the ideal softmax distribution leads to faster convergence and improved generalization. This insight motivates our codeword learning strategy, which minimizes Kullback-Leibler (KL) divergence to enhance the overall learning process. Extensive experiments on large-scale language models, sequential recommenders, and extreme multi-class classification tasks confirm that the \texttt{MIDX}-Sampler delivers superior effectiveness and efficiency compared to existing approaches, making it a compelling solution for large-scale machine learning problems.
\end{abstract}


\section{Introduction}
Machine learning techniques are widely used to extract patterns and capture the relationships between variables. One of the foundational problems in machine learning is the multi-class classification problem~\citep{aly2005survey,bishop2006pattern}, which arises in various domains such as image classification, natural language processing, information retrieval and recommender systems. The softmax function is commonly employed to estimate class probabilities based on computed logits from models. These predicted probabilities are then compared to true labels to compute the cross-entropy loss, and gradients over all candidate classes, derived from the partition function that sums over all logits, are computed to update model parameters.
However, this process introduces a significant computational burden, as the cost scales linearly with the number of classes. As a result, the softmax becomes computationally prohibitive when the number of classes grows very large.
Sampled softmax, which selects a fraction of samples to estimate the gradient over the whole corpus, garnered considerable attention due to its efficiency in handling large numbers of classes. The approach relies on self-normalized importance sampling to improve the approximation with lower variance, yet it still introduces a gradient bias compared to the original softmax. The only way to eliminate the bias is to use the true softmax distribution as the sampling distribution, which, however, is computationally inefficient.

Many previous studies have relied on simple, static sampling distributions, such as uniform or frequency-based distributions~\citep{mikolov2013distributed,bengio2008adaptive}, to reduce computational complexity. While these methods speed up the process, they often result in poor convergence, as the sampling distributions are significantly biased relative to the target softmax distribution. A more advanced approach, known as adaptive sampled softmax~\citep{bengio2008adaptive}, aims to address this by selecting a sampling distribution that closely approximates softmax, adapts to model changes, and prioritizes more challenging samples. However, this method becomes inefficient when applied to complex models, particularly in the case of millions or billions of classes. Numerous approaches have been proposed for dynamic negative sampling to enhance performance. MCMC-based methods~\citep{robert1999monte,bishop2006pattern}, for example, sample according to the long-term state transition matrix. However, these methods still suffer from high computational costs and the inevitable burn-in period. Other studies employ auxiliary networks to simulate the softmax distribution~\citep{salakhutdinov2015learning}, but the resulting sampling distributions are often intractable. More advanced techniques, such as those using quadratic kernels~\citep{blanc2018adaptive,yi2019sampling} to approximate softmax, provide poor approximations, limiting their effectiveness.

In light of these limitations, the need for more efficient sampling methods becomes clear.
In many machine learning models, logits are typically computed using a multi-layer perceptron (MLP), where the output layer’s dimensionality corresponds to the number of classes in the classification task. More broadly, the input to the softmax function can be viewed as the dot product between the query embedding and the candidate embeddings. The query embedding, often derived from sophisticated network architectures, captures intricate semantic or contextual information. The candidate embeddings, which represent individual classes, are usually stored in a cached embedding table for efficient retrieval.
By calculating the dot product between the query embedding and each candidate embedding, a similarity or relevance score is obtained, reflecting the relevance between the query and each class. These scores serve as the logits for the softmax function, which are later normalized as the class probabilities.
This concept reminds us of the maximum inner product search (MIPS) task~\citep{shrivastava2014asymmetric}, widely used in information retrieval and recommender systems. However, the key distinction is that MIPS algorithms deterministically return the same results for the same query, whereas sampling should introduce stochasticity based on sampling probabilities. One example is using the Gumbel-softmax trick ~\citep{jang2016categorical,lindgren2021efficient,mussmann2016learning} by adding random Guumbel variables to the embeddings to incorporate the stochasticity. But this approach still requires calculating over all classes for each sampling trial due to the change of the embeddings, which becomes inefficient for large-scale tasks.
Exploiting this connection with MIPS opens up possibilities for more efficient sampling by leveraging the similarities between class embeddings through efficient data index structures. Previous works have employed approximate nearest neighbor (ANN) methods, such as the Faiss library~\citep{douze2024faiss}, for quick retrieval and efficient sampling of similar class embeddings. Heuristic random sampling methods have also been incorporated to introduce uncertainty into the process. However, these methods do not derive the sampling distribution directly~\citep{spring2017new}, leading to inconsistencies in the self-normalized importance weights, which results in suboptimal performance.
Nonetheless, the objective of this paper is to reassess the relationship between indexing structures and sampling distributions, while achieving both efficiency and accurate approximation in sampling.

% \textcolor{red}{Disadvantage of Gumbel, How to convert to sampling}

In this way, we propose a novel sampler, \texttt{MIDX} Sampler, based on the advanced MIPS index, inverted multi-index. This approach allows us to achieve a less biased approximation of the softmax distribution and enables efficient sampling for the training of models. The quantization-based structure of the inverted multi-index facilitates the decomposition of the softmax probability into multiplications of multinomial probabilities, which enables independent sampling from multiple individual stages.
More specifically, except for the last stage, we sample a codeword index based on the previously sampled codeword. In the final stage, we select samples from the group of classes that are quantized to the sampled codewords. To further enhance sampling efficiency, we replace the query-adaptive distribution in the final stage with a static uniform distribution, where the sampling probability is based on the softmax normalization of the scores between queries and codewords. This substitution eliminates the need for recomputation during each training epoch while maintaining a relatively accurate estimation of the softmax distribution.
To better understand the impact of our sampler, we conduct a detailed theoretical analysis, examining the divergence between the proposal distributions and the target softmax distribution, as well as the gradient bias between the full softmax and the sampled softmax. Additionally, we analyze the convergence rate and the generalization error bound. Our findings show that smaller biases, as indicated by the KL-divergence between the proposal and the true softmax distribution, lead to faster convergence and improved generalization performance. This motivates the adoption of a specialized learning strategy for codewords, guided by minimizing the KL-divergence. Finally, we apply the \texttt{MIDX}-based samplers to various tasks, including natural language processing tasks, top-k recommenders and extreme classification tasks. Experimental results demonstrate the effectiveness and efficiency of the proposed samplers in these tasks. The implementation is available at \url{https://github.com/XuHwang/MIDX_Journal}.

Our main contributions can be summarized as follows:
\begin{itemize}[leftmargin=*,itemsep=0.85pt]
    \item We introduce a novel, precise estimation of the full softmax using quantization-based inverted multi-indexes, and prove that this method provides an exact approximation of the softmax distribution.
    \item We design a highly efficient sampling distribution based on decomposition, adopting a uniform distribution to reduce the computational cost from linear dependence on the number of classes to sub-linear dependence on the number of codewords.
    \item We provide a comprehensive theoretical comparison of the proposed sampling distributions with the ideal softmax distribution, including bounds on KL-divergence, gradient bias, convergence rates, and generalization error. To the best of our knowledge, this is the first work to offer such extensive theoretical analysis for different sampling distributions.
    \item We conduct extensive experiments across various tasks, including extreme classification, language models, and sequential recommenders. Our empirical results demonstrate the effectiveness, efficiency, and scalability of the \texttt{MIDX} sampler in practical applications.
\end{itemize}

% The rest of this paper is organized as follows: 

\section{Related Work} \label{sec:related_work}
% In this section, we review related work, including the approaches for efficient softmax calculation which includes different alternatives to the softmax, the maximum inner produce search problem which covers the different approximate nearest neighbor (ANN) index structures, and the negative sampling which includes various samplers.
In this section, we review related work, including the approaches for efficient softmax calculation, the maximum inner product search and the negative sampling techniques.

\subsection{Efficient Softmax}
Despite the appealing performance of softmax for classification problems, the bottleneck lies in the expensive computational cost of the partition function given the extremely significant number of classes. Hierarchical Softmax~\citep{goodman2001classes} is an early attempt to speed up the training by constructing tree structures, which reduces the time cost $\mathcal{O}(N)$ ($N$ the number of classes) to $\mathcal{O}(\log N)$.  Following that, deeper hierarchies~\citep{morin2005hierarchical, mnih2008scalable} are proposed to extend the approach and the Huffman coding based on word frequency is proposed for an optimal hierarchy~\citep{mikolov2013efficient}. Although these methods greatly reduce the computational complexity according to the tree structures, it is difficult to implement in the nowadays modern computations GPUs for fast matrix multiplications. Noise Contrastive Estimation~\citep{gutmann2010noise} avoids the softmax computation by converting the multi-classification problem into a binary logistic regression to distinguish the truth data from a generated noise one, and has been successfully adopted in language models~\citep{mnih2013learning}. 
Sparsemax~\citep{martins2016softmax} reduces the computational cost by assigning zero probabilities to some of the outputs and pushes the output to be more sparse. 

Sampling-based methods, which sample a fraction of classes to approximate the softmax function over all classes, have attained growing interest owing to their high efficiency~\citep{yang2024does}. Importance sampling, a Monte Carlo method, is popular to estimate the softmax function in the area of natural language processing~\citep{bengio2008adaptive} and formulates the sampled softmax. Tracing back to the sampling techniques in language models, the uniform distribution and unigram distribution are used to accelerate the training process~\citep{mikolov2013distributed}. With a deeper understanding of sampling methods, adaptive sampling methods, where more informative candidates are more likely to be chosen, have been proposed for better approximating the softmax. In recommender systems, the dynamic sampler is first proposed for Top-k collaborative filtering~\citep{zhang2013optimizing}, where the high-ranked items over a randomly sampled set would contribute more to the gradient descent.
% PRIS
Kernel-based methods choose kernel functions to estimate the exponential operation, \eg, quadratic kernel~\citep{blanc2018adaptive} and random fourier features kernel~\citep{rawat2019sampled}. This approach facilitates efficient computation over the entire output space through parallel-enabled tree structures.
% SVD-Softmax: Fast Softmax Approximation on Large Vocabulary Neural Networks

% joulin2017efficient

\subsection{Maximum Inner Product Search}
Maximum Inner Product Search (MIPS) is extensively employed for conducting similarity searches, enabling rapid top-k recommendations and document retrievals. Previous studies have addressed the MIPS problem by approximating or precisely reducing it to the nearest neighbor search (NNS). Notably, when the vectors possess identical norms, solving the MIPS problem is equivalent to performing NNS. In lower-dimensional spaces, exact solutions like kd-trees~\citep{bentley1975multidimensional} and R-trees~\citep{guttman1984r} can provide an exact nearest neighbor solution. However, these methods suffer from inefficient indexing when the dimensionality exceeds 20. To strike a balance between retrieval accuracy and search time in significantly higher dimensions, researchers have proposed well-designed indexes, including hashing, quantization, and graph-based methods~\citep{li2019approximate}. Local sensitive hashing methods~\citep{andoni2008near} hash similar points into the same bucket, thereby offering near-optimal guarantees for the tradeoff between index size and search time. Learning to hash~\citep{kulis2009learning} focuses on acquiring hashing functions that map high-dimensional data to binary hash codes. Quantization-based methods~\citep{jegou2010product,gersho2012vector} enhance semantic preservation by grouping similar data points into common codewords. Graph-based methods, such as HNSW~\citep{malkov2018efficient} and NSG~\citep{fu12fast}, achieve superior recall performance on real-world datasets by constructing hierarchical navigable graphs.

\section{Preliminaries}
% replace the cf model with sequential recommenders
% multi label for extreme class
Let us start with the classification problem with $N$ classes. Each sample $\bm{x}_i \in \mathcal{X}$ has one or multi label $y_i \in \mathcal{Y} = \{1, 2,..., N \}$ or $Y_i \subset \mathcal{Y}$. Each class has a one-hot embedding vector $\bm{y}_i \in \{0,1\}^N$. The classifier aims to predict the class given the query sample $\bm{x}_i$. The model first employs a mapping function $\phi_q: \mathcal{X} \rightarrow \mathbb{R}^D$ to generate the query or context embedding $\bm{z} = \phi_q(\bm{x}_i)$ and another mapping function $\phi_c: \mathcal{Y} \rightarrow \mathbb{R}^D$ to obtain the class embedding $\bm{q}_i = \phi_c(y_i)$. The model parameters involved in the mapping functions are summarized and denoted as $\bm{\theta}$. The prediction similarity score is calculated through the dot product $o_i = \bm{z}^\top  \bm{q}_i$. According to the predicted scores, a ranking list of classes is induced.

\subsection{Full Softmax}
The class probability distribution $\bm{p} \in [0,1]^N$ is computed by an exponential function:
\begin{displaymath}
\footnotesize
    p_i = \frac{ \exp o_i}{\sum_{j=1}^N \exp o_j}
\end{displaymath}
This refers to the \texttt{Softmax} function. The denominator corresponds to the \texttt{partition function}, and its time complexity $\mathcal{O}(N)$ increases linearly with the number of classes, referred to as the \textsc{full} softmax. The similarity score $o_i$ is often referred to as the logit. The cross-entropy loss is used to minimize the difference between the label $\bm{y}$ and the class probability distribution $\bm{p} $:
\begin{displaymath}
\footnotesize
    \ell (\bm{y}, \bm{p}) = - \sum_{i=1}^N y_i \log p_i = \log \sum_{j=1}^N \exp(o_j) -  o_i 
\end{displaymath}
Specifically, in order to train the neural network models, the gradient of the loss with respect to the model parameters $\bm{\theta}$ for each iteration is required, \ie,
\begin{displaymath}
\footnotesize
    \nabla_{\bm{\theta}} \ell(\bm{y}, \bm{p}) = - \nabla_{\bm{\theta}} o_i + \sum_{j=1}^N p_j \cdot \nabla_{\bm{\theta}} o_j = - \nabla_{\bm{\theta}} o_i + \mathbb{E}_{k\sim p}[\nabla_{\bm{\theta}} o_k]
\end{displaymath}
where the expectation is calculated over the softmax distribution $p$ and requires the enumeration over all classes. When the number of classes grows considerably large, the computational cost of learning the full softmax becomes expensive.

\subsection{Sampled Softmax}
%  QUESTION: when testing the model performance, the full softmax can not be avoided?  
% Answer: 推理的时候，需要计算相似度打分。一种是把所有类别的打分都计算出来进行排序，一种高效的办法是采用最大内积搜索或ANN的方法快速检索（通过一些索引结构）
The concept of sampled softmax was initially introduced to accelerate the training of language models~\citep{bengio2008adaptive}. Instead of the computation over all $N$ classes, only a sample of $M$ negative classes concerning the positive class are considered to approximate the \textsc{full} softmax. Given a training sample $\left(\bm{x}_i, y_i\right)$, a sample of $M$ negative classes is drawn according to the sampling distribution $q$. Assuming that $\mathcal{S}=\{s_1, s_2, ..., s_M\} \subset [N] \backslash \{i\}$ denotes the sampled set where $s_i$ refers to the class index, with the adjusted logits $\bm{o}' = \{o'_1, o'_2,...,o'_{M+1} \}$ and $o'_1 = o_i$, 
\begin{equation}
o'_{i+1} = 
\begin{cases}
    o_{s_i} - \ln(M q_{s_i}) & \text{if } y_{s_i} = 0 \\
    o_i & \text{else}
\end{cases}
\label{eq:correct_logit}
\end{equation}
the sampled softmax is unbiased in the limit $M \rightarrow \infty$~\citep{bengio2008adaptive}. The corrected softmax probability $\bm{p}'$ is calculated over the corrected logits:
\begin{displaymath}
\footnotesize
    p'_i = \frac{\exp (o'_i)}{ \sum_{j=1}^{M+1} \exp (o'_j)}
\end{displaymath}
Similarly, we can get the adjusted label vector $\bm{y}'$ by the projection from the original labels. Thus, the sampled softmax loss follows as
\begin{displaymath}
\footnotesize
    \ell' (\bm{p}', \bm{y}') = - \sum_{j=1}^{M+1} y'_j \log p'_j = \log \sum_{j=1}^{M+1} \exp(o'_j) -  o_i 
\end{displaymath}
In contrast to the \texttt{Full} softmax, the sampled softmax depends only on $M+1$ classes, which considerably reduces the computational cost $\mathcal{O}(N)$ to $\mathcal{O}(M)$. The gradient of the sampled softmax then follows as:
\begin{displaymath}
\footnotesize
    \nabla_{\bm{\theta}} \ell'(\bm{p}', \bm{y}') = - \nabla_{\bm{\theta}} o'_i + \sum_{j=1}^{M+1} p'_j \cdot \nabla_{\bm{\theta}} o'_j = - \nabla_{\bm{\theta}} o_i + \sum_{j=1}^{M+1} p'_j \cdot \nabla_{\bm{\theta}} o_j 
    % = - \nabla_{\bm{\theta}} o_i + \mathbb{E}_{k'\sim p'}[\nabla_{\bm{\theta}} o_{k'}]
\end{displaymath}
Despite the limited number of sampled negative classes, ideally, we would like to find a good sampling distribution to converge to the sample value as the \texttt{Full} softmax. If the sampled softmax is biased, \ie, 
\begin{displaymath}
\footnotesize 
    \mathbb{E}\left[ \nabla_{\bm{\theta}} \ell'(\bm{p}', \bm{y}') \right] \overset{?}{=} \nabla_{\bm{\theta}} \ell(\bm{y}, \bm{p})
    \quad \Rightarrow \quad
    \mathbb{E}\left[ \sum_{j=1}^{M+1} p'_j \cdot \nabla_{\bm{\theta}} o_j \right] \overset{?}{=} \sum_{j=1}^N p_j \cdot \nabla_{\bm{\theta}} o_j
\end{displaymath}
\cite{bengio2008adaptive} have highlighted that only the sampling distribution $q_i = p_i \propto \exp(o_i)$ achieves an unbiased estimator $\mathbb{E}\left[\nabla_{\bm{\theta}} \ell' \right] =\nabla_{\bm{\theta}} \ell$. Nevertheless, using such a sampling distribution remains computationally expensive during the training process, as it still requires $\mathcal{O}(N)$ time complexity for each training data point in every iteration. Increasing the number of negative samples $m$ can help alleviate sampling bias; however, an alternative distribution that is closer to the target distribution can be chosen to reduce the required sample size. Commonly used solutions include uniform and unigram distributions. However, these distributions are independent of the input and remain constant throughout the model training, leading to a significant bias in gradient estimation as they deviate considerably from the softmax distribution. Another noteworthy research direction focuses on the development of advanced adaptive samplers.
% The Gumbel trick has been employed to enable fast estimation by sampling the nearest neighbor~\citep{mussmann2017fast,ding2019fast}. 
The dynamic sampler, which selects high-ranked items from a randomly sampled set, has demonstrated superior performance in top-K item recommendation tasks~\citep{zhang2013optimizing}. The LSH-based sampler~\citep{vijayanarasimhan2014deep} selects nearest neighbors based on hashing buckets to improve the approximation in the embedding space.
The kernel-based approaches~\citep{blanc2018adaptive,rawat2019sampled} offer an alternative solution by directly estimating the exponential computation through kernel functions. These approaches leverage the decomposed nature of the kernel function, enabling efficient parallel computation. However, the approximation based on the kernel function still suffers from significant deviations. For instance, the quadratic kernel may assign higher values to negative logits~\citep{blanc2018adaptive}, while the RFF kernel requires additional normalization operations on the logits~\citep{rawat2019sampled}. Therefore, it is worth exploring the design of an adaptive sampler for the training of the softmax that achieves a more accurate approximation while making minimal changes to the logits.

\section{Sampling with Inverted Multi-Index}
Examining the calculation of logits, which can be decomposed into the query embedding $\bm{z}$ and class embedding $\bm{q}$, according to the sampling principle, where higher scores correspond with a higher sampling probability, the class exhibiting a higher inner product value with the given query embedding has a higher probability of being sampled. This characteristic establishes a connection with Maximum Inner Product Search (MIPS), which is further reduced to the approximate nearest neighbor (ANN) problem to get the solution. However, it is important to note that while MIPS consistently returns the same outcomes for identical queries, samplers yield stochastic results with inherent randomness for the identical query. This distinction motivates the design of samplers based on MIPS structures, \ie, the inverted multi-index, which offers a sub-linear time complexity for efficient sampling.

\begin{figure}[ht]
    \centering
    \includegraphics[width=0.85\textwidth]{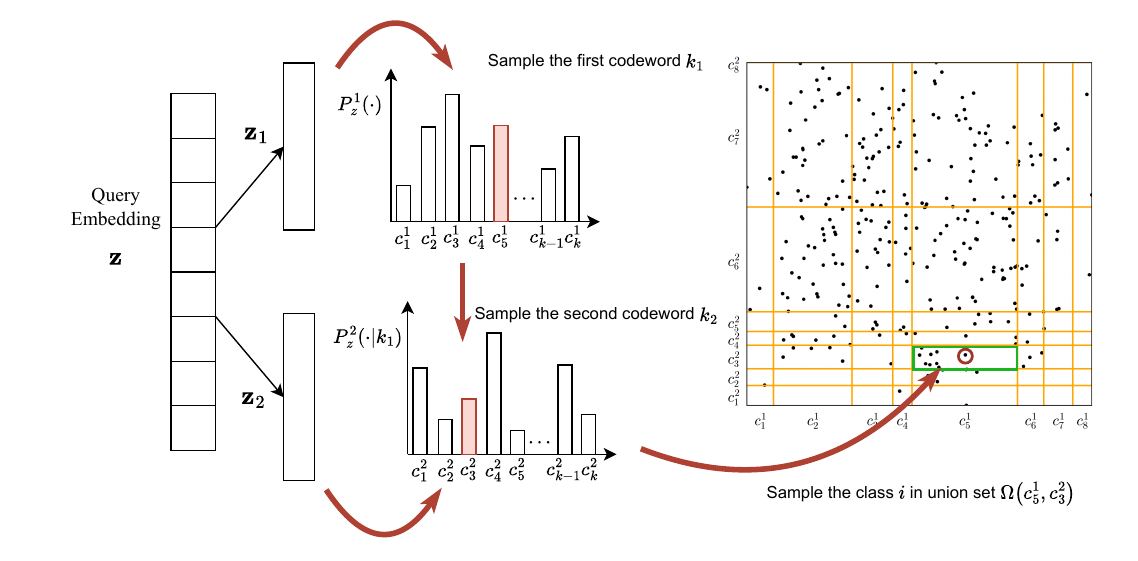}
    \vspace{-1.5em}
    \caption{Procedure of sampling a class given a query embedding through \texttt{MIDX} samplers. The subvectors $\bm{z}_1$ and $\bm{z}_2$ are derived depending on different quantizers, \eg, the product quantizers and the residual quantizers. The example follows as 1. Select the first codeword depends on the probability $P^1(\cdot)$; 2. Select the second codeword depends on the probability $P^2(\cdot|c_5^1)$; 3. Sample classes from the union set $\Omega(c_5^1, c_3^2)$, which includes the classes who are assigned to the 5-th codebook with the first subvector and to the 3-rd codebook with the second subvector.}
    \label{fig:illustration_midx_sampler}
\end{figure}

\subsection{Inverted Multi-Index}
The inverted index is a widely used structure in retrieval and ranking systems, particularly in scenarios involving a large number of candidates. Traditionally, this structure revolves around a \emph{codebook}—a collection of \emph{codewords} generated through clustering algorithms. The concept of the inverted multi-index~\citep{babenko2014inverted} builds upon this foundation by introducing \textit{multiple} codebooks, replacing the standard quantization method in the inverted index with the powerful technique of product quantization~\citep{jegou2010product}. This approach enables highly efficient similarity searches across divided subspaces, making it particularly well-suited for large-scale, high-dimensional data retrieval.

Constructing such an index involves two key steps: the choice of quantizer and the algorithm used to learn the codewords. Suppose there are $B$ codebooks, each containing $K$ codewords. Let $\mathcal{C}_l$ represent the $l$-th codebook, and let $\bm{c}_j^l$ denote the $j$-th codeword in $\mathcal{C}_l$. The quantizer aims to find the codeword closest to a given vector within each codebook. Given the vector $\bm{x}$, the quantizer function $f_q^l(\bm{x})$ is given by: $f_q^l(\bm{x}) = \mathop{\arg \min}_{\bm{c}_i^l \in \mathcal{C}_l} d (\bm{x}, \bm{c}_i^l)$, where $d (\bm{x}, \bm{c}_i^l)$ is the distance function, such as the Euclidean distance function. 
Regarding the learning of codewords, K-Means clustering is commonly employed, using all candidate vectors as input and producing the resulting clusters as codewords.

The samplers discussed aim to sample classes to estimate the softmax distribution, leveraging index structures built for the class embeddings $\{\bm{q}_i\}_{i=1}^N$. In the following, we explore two different quantizer methods in detail.
\begin{itemize}[leftmargin=*,itemsep=1pt,parsep=3pt]
    \item Product Quantization: This method forms the foundation of the standard inverted multi-index. It decomposes the vector space into $B$ subspaces and applies a codeword learning algorithm in each subspace to generate codewords. Specifically, the class embedding $\bm{q}$ is partitioned as $\bm{q} = [\bm{q}^1, \bm{q}^2, \dots, \bm{q}^B]$, where each $\bm{q}^l$ represents a subvector with dimensionality $\frac{D}{B}$, and $[,]$ denotes the concatenation operation. For the $l$-th codebook $\mathcal{C}_l$, the codewords $\{\bm{c}_j^l \in \mathbb{R}^{\frac{D}{B}}\}_{j=1}^K$ are generated using the K-Means clustering algorithm, with the subvectors $\{\bm{q}i^l\}_{i=1}^N$ as input.
    \item Residual Quantization: It improves the quantization quality by considering the residuals after the initial quantization. Concretely, to build the $l$-th codebook $\mathcal{C}_l$, the residual vectors are calculated as $\bm{q} -  \sum^{l-1}_{j=1} \bm{c}_{k_j}^j $, denoting the differences between the original vector $\bm{q}$ and the reconstructed vector. Then all residual vectors of classes are input into the K-Means clustering to get the codewords, where the codeword $\bm{c}_j^l$ is in $D$ dimension.
 \end{itemize}
The final reconstructed vector is expressed as $[\bm{c}{k_1}^1 \oplus \bm{c}{k_2}^2 \oplus \dots \oplus \bm{c}_{k_B}^B]$, where $\oplus$ represents concatenation in product quantization and addition in residual quantization.

After constructing the index for the inverted multi-index, the similarity search becomes highly efficient. In each codebook, only the closest codeword needs to be identified, reducing the time complexity from $\mathcal{O}(N)$ (where $N$ is the total number of classes) to $\mathcal{O}(K)$ for each of the $B$ steps. Increasing the number of codebooks would enhance search accuracy but comes at the expense of retrieval speed. In previous multi-index cases, it has been suggested that vectors be split into two blocks, imposing stricter constraints on query time. For the sake of simplicity, we elaborate on the subsequent discussion with the assumption of two codebooks, \ie, $B=2$.

Inspired by the elegance and efficiency of this similarity search technique, we propose the design of the \texttt{MIDX} sampler, aiming to incorporate the principles of the inverted multi-index to improve the accuracy and computational efficiency of the softmax training process.

\subsection{Exact MIDX Sampler}
Although several studies leverage the advantages of ANN structures for dynamic sampling~\citep{shrivastava2014asymmetric,mussmann2016learning}, challenges such as intractable sampling distributions or significant bias from the original softmax distribution persist. These issues lead to inconsistencies in the self-normalized estimator, ultimately resulting in suboptimal performance.
However, by leveraging the structure of the inverted multi-index, it is possible to decompose the softmax probability into the multiplication of multiple multinomial probabilities. This decomposition leads to a sub-linear time complexity for softmax training, as stated in the following theorem. Typically, the logit $o_i$ is often calculated through the dot product between the query embedding $\bm{z}$ and the class embedding $\bm{q}_i$. We consider the connection with query embedding to note the softmax probability as $P(i|\bm{z})$.
\begin{mytheorem} \label{theo:midx_exact}
    Given the query embedding $\bm{z} = [\bm{z}^1 \oplus \bm{z}^2]$, the $i$-th class embedding is denoted as $\bm{q}_i = [\bm{c}^1_{k_2} \oplus \bm{c}^2_{k_2}] + \tilde{\bm{q}}_i$, where $\tilde{\bm{q}}_i$ denotes the residual vector. $\Omega_{k_1, k_2}$ denotes the set of classes grouped to $\bm{c}^1_{k_1}$ in the first codebook and grouped to $\bm{c}^2_{k_2}$ in the second codebook. The softmax probability $P(i|\bm{z}) = \frac{\exp (\bm{z}^\top \bm{q}_i ) }{\sum_{j=1}^N \exp (\bm{z}^\top \bm{q}_j)}$ can be decomposed as 
    \begin{gather}
        P(i|\bm{z}) = \frac{\exp (\bm{z}^\top \bm{q}_i ) }{\sum_{j=1}^N \exp (\bm{z}^\top \bm{q}_j)} = P^1_{\bm{z}} (k_1) \cdot P^2_{\bm{z}}(k_2 | k_1) \cdot P^3_{\bm{z}}(i|k_1, k_2), \label{eq:midx_dec}
        \end{gather}
        \begin{gather}
        P^1_{\bm{z}} (k_1) = \frac{ \psi_{k_1} \exp ( {\bm{z}^1}^\top \bm{c}^1_{k_1}) }{\sum_{k=1}^K \psi_{k} \exp ( {\bm{z}^1}^\top \bm{c}^1_{k}) } \label{eq:prob_c1}\\
        P^2_{\bm{z}}(k_2 | k_1) = \frac{ \omega_{k_1, k_2} \exp ( {\bm{z}^2}^\top \bm{c}^2_{k_2})}{ \underbrace{\sum_{k'=1}^K \omega_{k_1, k'} \exp ({\bm{z}^2}^\top \bm{c}^2_{k'})}_{\psi_{k_1}}}  \label{eq:prob_c2}\\
        P^3_{\bm{z}}(i|k_1, k_2) = \frac{\exp (\bm{z}^\top \tilde{\bm{q}}_i)}{\underbrace{\sum_{j \in \Omega_{k_1,k_2}} \exp (\bm{z}^\top \tilde{\bm{q}}_j)}_{\omega_{k_1, k_2}}} \label{eq:union_c1c2}
    \end{gather}
\end{mytheorem}

\begin{proof}
    \begin{displaymath}
    \small
        \begin{split}
            P(i|\bm{z}) & = \frac{\exp (\bm{z}^\top \bm{q}_i ) }{\sum_{j=1}^N \exp (\bm{z}^\top \bm{q}_j)} = \frac{ \exp ({\bm{z}^1}^\top \bm{c}^1_{k_1} )  \exp ( {\bm{z}^2}^\top \bm{c}^2_{k_2}) \exp (\bm{z}^\top \tilde{\bm{q}}_i)}{ \sum_{j=1}^N \exp ({\bm{z}^1}^\top \bm{c}^1_{k'_1} )  \exp ( {\bm{z}^2}^\top \bm{c}^2_{k'_2}) \exp (\bm{z}^\top \tilde{\bm{q}}_j)} \\
            & =  \frac{ \exp ({\bm{z}^1}^\top \bm{c}^1_{k_1} )  \exp ( {\bm{z}^2}^\top \bm{c}^2_{k_2}) \exp (\bm{z}^\top \tilde{\bm{q}}_i)}{ \sum_{k=1}^K \exp ({\bm{z}^1}^\top \bm{c}^1_k) \underbrace{\sum_{k'=1}^K \exp ({\bm{z}^2}^\top \bm{c}^2_{k'}) \sum_{j \in \Omega_{k,k'}} \exp (\bm{z}^\top \tilde{\bm{q}}_j)}_{\psi_k}} \\
            & = \frac{ \psi_{k_1} \exp ( {\bm{z}^1}^\top \bm{c}^1_{k_1}) }{\sum_{k=1}^K \psi_{k} \exp ( {\bm{z}^1}^\top \bm{c}^1_{k}) } \cdot \frac{\exp ( {\bm{z}^2}^\top \bm{c}^2_{k_2}) \exp (\bm{z}^\top \tilde{\bm{q}}_i)}{\psi_{k_1}} \\
            & =  P^1_{\bm{z}} (k_1) \cdot \frac{\exp ( {\bm{z}^2}^\top \bm{c}^2_{k_2}) \exp (\bm{z}^\top \tilde{\bm{q}}_i)}{\sum_{k'=1}^K \exp ({\bm{z}^2}^\top \bm{c}^2_{k'}) \underbrace{\sum_{j \in \Omega_{k_1,k'}} \exp (\bm{z}^\top \tilde{\bm{q}}_j)}_{\omega_{k_1, k'}}} \\
            & = P^1_{\bm{z}} (k_1) \cdot \frac{ \omega_{k_1, k_2} \exp ( {\bm{z}^2}^\top \bm{c}^2_{k_2})}{\sum_{k'=1}^K \omega_{k_1, k'} \exp ({\bm{z}^2}^\top \bm{c}^2_{k'})} \cdot \frac{\exp (\bm{z}^\top \tilde{\bm{q}}_i) }{\omega_{k_1, k_2}} \\
            & = P^1_{\bm{z}} (k_1) \cdot P^2_{\bm{z}}(k_2 | k_1) \cdot \frac{\exp (\bm{z}^\top \tilde{\bm{q}}_i)}{\sum_{j \in \Omega_{k_1,k_2}} \exp (\bm{z}^\top \tilde{\bm{q}}_j)}
            = P^1_{\bm{z}} (k_1) \cdot P^2_{\bm{z}}(k_2 | k_1) \cdot P^3_{\bm{z}}(i|k_1, k_2)
        \end{split}
    \end{displaymath}
\end{proof}

This theorem can be straightforwardly extended to the case with more codebooks, \ie, $B>2$. This provides a novel perspective to exactly sample a candidate from the softmax probability. Here we take the example with two codebooks to illustrate the sampling process. As illustrated in Figure~\ref{fig:illustration_midx_sampler}, the \texttt{MIDX} sampler has the following steps:
\begin{enumerate}[itemsep=1pt,parsep=2pt]
    \item Construct the three multinomial probabilities, \ie, Eq~\eqref{eq:prob_c1}\eqref{eq:prob_c2}\eqref{eq:union_c1c2}. 
    \item Sample the index $k_1$ from the first codebook $\mathcal{C}^1$ according to the probability $P^1_{\bm{z}}(\cdot)$. 
    \item  Sample another index $k_2$ from the second codebook $\mathcal{C}^2$ according to the conditional probability $P^2_{\bm{z}}(\cdot|k_1)$. 
    \item  Sample a candidate class according to the residual softmax probability $P^3_{\bm{z}}(\cdot |k_1, k_2)$.
\end{enumerate}
Compared with the naive inverted multi-index, the \texttt{MIDX} sampler ensures that a class would be successfully sampled at each trial in the last step. The reason lies in that the weights $w_{k_1,k_2}$ would always be positive, and the empty union set $\Omega$ would be discarded during the sampling process due to the zero probability. This guarantees that only non-empty sets are considered during sampling.

\begin{algorithm}[th]
	\caption{MIDX Sampler}
	\label{alg:sampling}
	\LinesNumbered
	\KwIn{Class embeddings $\{\bm{q}_i \in \mathbb{R}^D| i \in \mathcal{Y}\}$, query embedding $\bm{z}$, sampling size $M$, codebook size $K$}
	\KwOut{Sample set $\Phi$ including the class index and the sampling probability}
	% \tcp{Initialization part in $\mathcal{O}(KNDt)$} 
    % \tcp{$t$ is the iteration number in K-means to obtain the codewords}
	$\mathcal{C}^{1},\mathcal{C}^{2}\leftarrow$ ProductQuantization($\{\bm{{q}}_i| i\in \mathcal{I}\}$, $K$) or ResidualQuantization \;
	Compute residual vectors for all classes $\{\tilde{\bm{q}}_i = \bm{q}_i - [\bm{c}_{k}^1\oplus\bm{c}_{k'}^2]| i\in \mathcal{I}\}$\;
	Compute $\Omega_{k_1,k_2} = \{ i|k=k_1, k'=k_2, i\in\mathcal{Y} \}, \forall 1\le k_1, k_2 \le K$ \;
	\tcp{Sampling part in $\mathcal{O}(ND + M)$}
	\For{$k_1=1$ \KwTo $K$}{
		\For{$k_2=1$ \KwTo $K$}{
			Compute $\omega_{k_1, k_2}$ and construct $P_u^3(\cdot|k_1,k_2)$ in Eq~\eqref{eq:union_c1c2}\;
		}
		Compute $\psi_{k_1}$ and construct $P_u^2(\cdot|k_1)$ in Eq~\eqref{eq:prob_c2}\;
	}
	Construct $P_u^1(\cdot)$ in Eq~\eqref{eq:prob_c1}\;
	Initialize $\Phi=\emptyset$\;
	\For{$i=1$ \KwTo $M$}{
		Respectively sample $k_1, k_2, i$ from $P_u^1(\cdot), P_u^2(\cdot|k_1)$ and $P_u^3(\cdot|k_1,k_2)$ in order\;
		$Q(y_i|\bm{z})\leftarrow P_u^1(k_1)P_u^2(k_2|k_1)P_u^3(y_i|k_1,k_2)$\;
		$\Phi \leftarrow \Phi \cup (i, Q(y_i|\bm{z}))$\;
	}
	Return $\Phi$\;
\end{algorithm}
\noindent \textbf{Time complexity analysis.}
As shown in Algorithm~\ref{alg:sampling}, the overall procedure can be split into two parts. Lines 1-3 describe the initialization part to obtain codebooks and lines 4-13 describe the sampling part with the computation of the probability. Being independent to queries, the initialization part is only executed once in $\mathcal{O}(KNDt)$, where $N$ is the number of classes, $D$ is the dimension of embeddings and $t$ is the number of iterations in K-means. Thanks to the Vose-Alias method sampling techniques~\citep{walker1977efficient}, the sampling part only takes $\mathcal{O}(1)$ time to sample an item. Unfortunately, it is necessary to compute the inner-product logits over all items, which takes $\mathcal{O}(ND)$ time. This indicates that it is no more efficient than sampling from the softmax distribution directly.

\subsection{Faster MIDX-based Sampler}
The inefficiency of the exact \texttt{MIDX} sampler comes from the query-based computation of the inner product of the query and the residual vectors among all classes, \ie, $P_{\bm{z}}^3(i|k_1, k_2)$, which takes a time complexity of $\mathcal{O}(ND)$. This still maintains the same complexity as the softmax, preventing quick computation. Thus, we provide the variants of the \texttt{MIDX} sampler, where the static distribution, \ie, uniform sampling, is chosen to replace the probability $P^3_{\bm{z}}(\cdot|k_1,k_2)$. The multinomial distribution $P^1(\cdot)$ and $P^2(\cdot|k_1,k_2)$ can be efficiently solved since they only involve the inner product between query and codewords rather than the whole corpus, achieving a sub-linear complexity.
% We choose the uniform distribution to replace the probability, where each class within the union set has the same probability of being sampled, and we name it as the \texttt{MIDX} sampler. 
The proposal distribution can be derived based on the following theorem.
\begin{mytheorem}
    In Theorem~\ref{theo:midx_exact}, replace $P^3_{\bm{z}}(\cdot|k_1, k_2) $ with a uniform distribution, \ie, $P^3_{\bm{z}}(i|k_1, k_2)= \frac{1}{\vert \Omega_{k_1, k_2} \vert } $, and revise $\omega_{k_1, k_2} = \vert \Omega_{k_1, k_2} \vert $. The proposal distribution for MIDX behaves:
    \begin{equation}
        \begin{split}
            Q_{\text{midx}}(i |\bm{z}) & = \frac{ \exp ({\bm{z}^1}^\top \bm{c}^1_{k_1} )  \exp ( {\bm{z}^2}^\top \bm{c}^2_{k_2}) }{ \sum_{k=1}^K \exp ({\bm{z}^1}^\top \bm{c}^1_k) \sum_{k'=1}^K \vert \Omega_{k,k'} \vert \exp ({\bm{z}^2}^\top \bm{c}^2_{k'}) } \\
            & = \frac{ \exp \left( \bm{z}^\top (\bm{q}_i - \tilde{\bm{q}}_i) \right)}{ \sum_{j\in \mathcal{I}} \exp \left( \bm{z}^\top (\bm{q}_j - \tilde{\bm{q}}_j) \right)} = \frac{\exp (o_i - \tilde{o}_i)}{\sum_{j\in \mathcal{I}} \exp (o_j - \tilde{o}_j)}
        \end{split}
    \label{eq:midx_uni}
    \end{equation}
\end{mytheorem}

\begin{proof}
    \begin{displaymath}
        \begin{split}
            Q_{\text{midx}}(i |\bm{z}) & = P^1_{\bm{z}} (k_1) \cdot P^2_{\bm{z}}(k_2 | k_1) \cdot P^3_{\bm{z}}(i|k_1, k_2) \\
    \end{split}
    \end{displaymath}
    \begin{displaymath}
        \begin{split}
            & = \frac{ \psi_{k_1} \exp ( {\bm{z}^1}^\top \bm{c}^1_{k_1}) }{\sum_{k=1}^K \psi_{k} \exp ( {\bm{z}^1}^\top \bm{c}^1_{k}) } \cdot \frac{ \omega_{k_1, k_2} \exp ( {\bm{z}^2}^\top \bm{c}^2_{k_2})}{ \underbrace{\sum_{k'=1}^K \omega_{k_1, k'} \exp ({\bm{z}^2}^\top \bm{c}^2_{k'})}_{\psi_{k_1}}}  \cdot \frac{1}{\underbrace{\vert \Omega_{k_1,k_2} \vert }_{\omega_{k_1, k_2}}} \\
            & = \frac{  \exp ( {\bm{z}^1}^\top \bm{c}^1_{k_1}) \exp ( {\bm{z}^2}^\top \bm{c}^2_{k_2})}{\sum_{k=1}^K \sum_{k'=1}^K|\Omega_{k,k'}| \exp ( {\bm{z}^1}^\top \bm{c}^1_{k}) \exp ({\bm{z}^2}^\top \bm{c}^2_{k'}) } \\
            & = \frac{\exp \left( \bm{z}^\top (\bm{q}_i - \tilde{\bm{q}}_i)\right)}{\sum_{j\in \mathcal{I}} \exp \left( \bm{z}^\top (\bm{q}_j - \tilde{\bm{q}}_j)\right)} = \frac{\exp (o_i - \tilde{o}_i)}{\sum_{j\in \mathcal{I}} \exp (o_j - \tilde{o}_j)}
        \end{split}
    \end{displaymath}
\end{proof}
According to the theorem, those codewords having greater values of inner products will have higher probabilities of being sampled. This behavior arises from the grouping of similar items with the same codewords. Consequently, items contained within high-scoring codewords are more likely to exhibit higher logits and have increased probabilities of being sampled, which is consistent with the original softmax distribution.

\noindent \textbf{Time complexity analysis.}
When computing the sampling probability, the computation only involves the inner product between the query embedding and all codewords, which takes $\mathcal{O}(KD)$ to compute. In addition, it takes $\mathcal{O}(K^2)$ since it should calculate the normalization constant in $P_2(\cdot|k_1)$ for each $k_1$. Overall, the time complexity of the preprocessing part is $\mathcal{O}(KD+K^2)$. Since the codebook size $K$ is much smaller than the number of items $N$, the \texttt{MIDX} sampler is much more efficient than the original exact sampling.

\subsection{Complexity Analysis}
Table~\ref{tab:time_complexity} summarizes the time and space complexity for sampling from different proposals, which demonstrates the superiority of the \texttt{MIDX} sampler in space and time cost. Thanks to the independence of the queries, the \texttt{MIDX} sampler can be implemented on the GPUs, which accelerates the sampling procedure. 
%and does not cost much time compared with the training process.
Note that the initialization time includes the index construction, such as alias tables, inverted multi-index or tree. The initialization is only updated before each epoch and does not cost much time compared with the training process. Regarding the space complexity, the \texttt{MIDX} sampler just additionally stores the codewords and the indices, which is much smaller than that of the advanced adaptive kernel-based samplers, \ie, Quatratic Kernel Sampler or Random Fourier Feature Kernel Sampler (RFF-Kernel). Furthermore, the sampling time complexity of the \texttt{MIDX} sampler scales linearly to the number of the codewords $K$, rather than the number of classes $N$ or the logarithm $\log N$, being significantly efficient than kernel-based samplers. 

% \subsection{Consistent Updating with Learnable Codebooks}
% \textcolor{red}{update frequency, check for the results and then add it or not}

\begin{table}[t]
	\centering
	\begin{tabular}{c|c|c|c}
		\toprule
		Proposals $Q$ & Space & Init Time &  Sample Time\\
		\midrule
		Uniform & $1$  & - &  $M$ \\
		Unigram & $N$ & $N$ & $M$  \\
		Quadratic Kernel~\citep{blanc2018adaptive} & $ND^2$ & $ND^2$ & $D^2 M \log N$  \\
        RFF Kernel~\citep{rawat2019sampled} & $NRD$ & $NRD$ & $RM\log N$ \\
		Exact MIDX in Eq~\eqref{eq:midx_dec} & $ND$ & $KNDt$ & $ND + M$ \\
		MIDX in Eq~\eqref{eq:midx_uni} & $KD +K^2+ N$ & $KNDt$ & $KD + K^2 + M$ \\
		% MIDX-pop in Eq. & $Kd +K^2+ N$ & $KNdt$ & $Kd + K^2 + T$\\
		\bottomrule
	\end{tabular}
        \vspace{-1.2em}
        \caption{Time and space complexity of sampling $M$ classes with different proposals. Denote by $N$ the class number, $D$ the embedding dimension, and $K$ the codebook size, $t$ the number of iterations in K-means, $R$ the dimension of RFF map. ($K, D\ll N$)}
	\label{tab:time_complexity}
\end{table}

\section{Theoretical Analysis}
In this section, we provide more theoretical analysis for different proposals about the sampling distribution bias, convergence rate and generalization error bound to demonstrate the effectiveness of the proposed \texttt{MIDX}-sampler. 
% For the sake of simplicity, the proof can be attached in the Appendix.

\subsection{Bias from Softmax Distribution}\label{sec:theo_distri_bias}
Firstly, we theoretically explore the bias of the proposed distribution from the softmax distribution according to the KL-divergence
% , $ \mathcal{D}_{KL} \left[Q(\cdot |\bm{z}) \Vert P(\cdot |\bm{z})\right]=\sum_{i=1}^N Q(i|\bm{z}) \log \frac{Q(i|\bm{z})}{P(i|\bm{z})}$, where $P(i|\bm{z})=\frac{\exp (\bm{z}^\top \bm{q}_i)}{\sum_{j=1}^N \exp(\bm{z}^\top \bm{q}_j) }$ denotes the ideal softmax probability and $P(\cdot|\bm{z})$ denotes the softmax distribution over $N$ classes. $Q(i|\bm{z})$ denotes the sampling proposal with respect to the class $i$. 
% $Q(i|\bm{z})$ and $Q(\bm{i} | \bm{z})$ denote the sampling probability and distortion respectively.
% Now, we vary the samplers to derive their bias of sampling distribution from the softmax distribution.

\subsubsection{Uniform proposal}
\begin{mytheorem}\label{theorem:kl_unif}
    Assume $o_i = \bm{z}^\top \bm{q}_i$ denotes the similarity score with respect to the query $\bm{z}$ and the $i$-th class and $\bm{o}\in \mathbb{R}^{N}$ is a vector involving all similarity scores over all classes. The KL-divergence from the target softmax distribution $P(\cdot|\bm{z})$ to the proposal distribution $Q_{\text{uniform}} (\cdot | \bm{z})$ according to the uniform sampling can be bounded from below:
    \begin{displaymath}
        0 \le \mathcal{D}_{KL}\left[   Q_{\text{uniform}} (\cdot | \bm{z}) \Vert P  (\cdot | \bm{z}) \right] \leq 2 \Vert \bm{o} \Vert_{\infty}.
    \end{displaymath}

\end{mytheorem}

\subsubsection{Unigram Proposal}
\begin{mytheorem}\label{theorem:kl_unig}
    Assume $o_i = \bm{z}^\top \bm{q}_i$ denotes the similarity score with respect to the query $\bm{z}$ and the $i$-th class and $\bm{o}\in \mathbb{R}^{N}$ is a vector involving all similarity scores over all classes. $q_{min}$ and $q_{max}$ denote the minimal and maximal probability according to the normalized unigram distribution given the data distribution.
    The KL-divergence from the softmax distribution $P(\cdot|\bm{z})$ to the proposal distribution $Q_{\text{unigram}} (\cdot| \bm{z})$ according to the unigram sampling can be bounded from below:
    \begin{displaymath}
        0 \leq \mathcal{D}_{KL}\left[   Q_{\text{unigram}} (\cdot | \bm{z}) \Vert P  (\cdot | \bm{z}) \right] \leq 2 \Vert \bm{o} \Vert_{\infty} + \ln N q_{max}
    \end{displaymath}
\end{mytheorem}

\subsubsection{MIDX Proposal}
\begin{mytheorem}\label{theorem:kl_midx_uni}
    Assuming that the $\tilde{o}_i = \bm{z}^\top \tilde{\bm{q}}_i$ is the similarity score with respect to the residual embedding $ \bm{\tilde{q}}_i$ and the vector $\tilde{\bm{o}}\in \mathbb{R}^N$ denotes the score vector over all classes, the KL divergence from the softmax distribution $P(\cdot|\bm{z})$ to the proposal distribution $Q_{\text{midx}} (\cdot| \bm{z})$ according to the MIDX sampler can be bounded from below:
	\begin{displaymath}
		0 \le \mathcal{D}_{KL}\left[   Q_{\text{midx}}(\cdot|\bm{z}) \Vert P(\cdot|\bm{z}) \right] \le 2  \Vert \tilde{\bm{o}} \Vert_\infty.
	\end{displaymath}
\end{mytheorem}
According to the divergence of the proposal from the \texttt{MIDX} depends on the residual similarity score, \ie, $\tilde{\bm{o}}$. 
The infinite norm of the residual vector corresponds to the distortion rate in quantization. The distortion rate $E$ has a lower bound, as discussed in previous work~\citep{ge2013optimized}, given by:
\begin{displaymath}
    E = \sum_{i=1}^N \Vert \tilde{\bm{q}}_i \Vert_2^2 \leq K^{- \frac{2}{D} } D \vert \Sigma \vert^{\frac{1}{D}}
\end{displaymath}
where $K$ denotes the number of codewords and $D$ denotes the embedding dimension. $\Sigma$ is the determinant from the Gaussian assumption. This lower bound provides a guarantee that with the increasing granularity of space partition (the number of codewords in each codebook), the residual vectors are of small magnitude such that the upper bound for the bias becomes smaller.

For comparison, the KL-divergence of different sampling proposals from the softmax distribution is summarized in Table~\ref{tab:kl_samplers}. Both the Uniform and Unigram samplers have a divergence bound related to the infinity norm of the logits. In contrast, the \texttt{MIDX} sampler’s bound is associated with the norm of the residual logits, which is smaller than that of the original logits after quantization. This indicates that the \texttt{MIDX} sampler has a lower KL divergence, leading to more accurate approximations of the softmax distribution.

\begin{table}[t]
    \centering
    \begin{tabular}{c|c|c}
        \toprule
        Sampler & Sampling Probability $Q(i|\bm{z})$ & Upper Bound \\
        \midrule
        Uniform & $\frac{1}{N}$ & $2 \Vert \bm{o} \Vert_{\infty}$ \\
        Unigram & $q(i)$ & $2 \Vert \bm{o} \Vert_{\infty} + \ln N  q_{max}$ \\
        MIDX & $\frac{\exp (o_i - \tilde{o}_i)}{\sum_{j=1}^N \exp (o_j -  \tilde{o}_j)}$ & $ 2  \Vert \tilde{\bm{o}} \Vert_\infty$ \\
        \bottomrule
    \end{tabular}
    \caption{KL-divergence of different samplers}
    \label{tab:kl_samplers}
\end{table}

\subsection{Gradient Bias}\label{sec:theo_gradient_bias}
In this section, we provide another important theoretical analysis for different samplers, \ie, gradient error. Before delving into the technical details, we state the following important assumptions for the encoders of queries and classes.

\begin{myassumption}\label{theorem:assump_gradient_bias}
    The following conditions hold for the query encoder and classes encoder:

    1. The encoder functions (mapping functions $\phi_c$ and $\phi_q$) are L-Lipschiz in the parameter $\theta$. In particular, we have $\Vert \phi\left( \theta\right) - \phi \left( \theta'\right) \Vert_2 \leq L \Vert \theta - \theta' \Vert_2$ for $\theta$, $\theta'$.

    2. The logits have bounded gradients, \ie, we have $\Vert \nabla o_{j} \Vert_2 \leq U$ for all $j \in [N]$.
\end{myassumption}
All of the assumptions are fairly mild and common in the literature. In the following, we first show that the error can be bounded by the number of samples and the divergence from the target softmax distribution $P$ and the proposal distribution $Q$.

\begin{mytheorem}\label{theorem:gradient_error_general} (Proposition 7 in~\citep{metelli2020importance})
    Let $s_1, s_2, ..., s_M$ i.i.d. random variables sampled from proposal $Q$, the gradient bias follows:
\begin{displaymath}
\begin{split}
    \vert \mathbb{E}[\nabla_{\theta_t} \ell'] - \nabla_{\theta_t} \ell \vert 
    & \le U \min \left\{2, \sqrt{\frac{d_2(P \Vert Q) - 1}{M + 1}}\right\}
\end{split}
\end{displaymath}
where $d_2(P\Vert Q) = \mathbb{E}_{i\sim P} [p_i/ q_i]$ denotes the exponential of the second-order Renyi divergences, which evaluates the difference from the target distribution. 
\end{mytheorem}

According to the theorem, the upper bound of the gradient bias is determined by the divergence between the softmax distribution and the sampling proposals. As the sampling proposal converges to the target softmax distribution, the gradient bias decreases. In the following, we explore different samplers to derive the specific bounds for each approach.

\subsubsection{Uniform Sampler}
\begin{mytheorem}\label{theorem:gradient_error_uni}
    Let $s_1, s_2, ..., s_M$ i.i.d. random variables sampled from uniform proposal $Q$, the gradient approximation is bounded by:
    \begin{displaymath}
    \left\vert \mathbb{E}[\nabla_{ \theta_t} \ell'] - \nabla_{ \theta_t}\ell \right\vert  \le \min \left\{2, U \sqrt{\frac{\exp \left(2 \Vert \bm{o} \Vert_\infty \right) - 1 }{M + 1}} \right\}
\end{displaymath}
\end{mytheorem}

\subsubsection{Unigram Sampler}
\begin{mytheorem}\label{theorem:gradient_error_pop}
    Let $s_1, s_2, ..., s_M$ i.i.d. random variables sampled from unigram proposal $Q$, the gradient approximation is bounded by:
    \begin{displaymath}
    \left\vert \mathbb{E}[\nabla_{ \theta_t} \ell'] - \nabla_{ \theta_t}\ell \right\vert \le \min \left\{2, U \sqrt{\frac{\exp \left(2 \Vert \bm{o} \Vert_\infty - \ln q_{min}\right) - 1 }{M + 1}} \right\} 
\end{displaymath}
\end{mytheorem}
where $q_{max}$ and $q_{min}$ denote the maximum and minimal value for unigram distribution according to different data distributions.

It can be observed that if the frequency of classes varies greatly, the bound of the gradient error would get larger. Thus, more balanced classes are beneficial for better convergence.

\subsubsection{MIDX Sampler}
\begin{mytheorem}\label{theorem:gradient_error_midx}
    Let $s_1, s_2, ..., s_M$ i.i.d. random variables sampled from MIDX proposal $Q$, the gradient approximation is bounded by:
    \begin{displaymath}
    \left\vert \mathbb{E}[\nabla_{ \theta_t} \ell'] - \nabla_{ \theta_t}\ell \right\vert \le \min \left\{2, U \sqrt{\frac{\exp (2 \Vert \tilde{\bm{o}} \Vert_\infty) - 1 }{M + 1}} \right\} 
\end{displaymath}
\end{mytheorem}

Compared with the uniform sampler, the gradient bias is correlated with the residual vector, which depends on the different quantizers. For the sake of simplicity, we summarize the upper bound of the gradient approximation for different samplers in Table~\ref{tab:grad_bias_samplers}.

\begin{table}[t]
    \centering
    \begin{tabular}{c|c|c}
        \toprule
        Sampler & Sampling Probability $Q(i|\bm{z})$ & Upper Bound \\
        \midrule
        Uniform & $\frac{1}{N}$ & $U \sqrt{\frac{\exp \left(2 \Vert \bm{o} \Vert_\infty \right) - 1 }{M + 1}}$ \\
        Unigram & $q(i)$ & $ U \sqrt{\frac{\exp \left(2 \Vert \bm{o} \Vert_\infty - \ln q_{min}\right) - 1 }{M + 1}}$ \\
        MIDX & $\frac{\exp (o_i - \tilde{o}_i)}{\sum_{j=1}^N \exp (o_j -  \tilde{o}_j)}$ & $ U \sqrt{\frac{\exp (2 \Vert \tilde{\bm{o}} \Vert_\infty) - 1 }{M + 1}}$ \\
        \bottomrule
    \end{tabular}
    \caption{Gradient Approximation of different samplers.}
    \label{tab:grad_bias_samplers}
\end{table}

\subsection{Convergence Rate}~\label{sec:convergence_rate}
We further theoretically analyze the convergence rate of the different samplers to demonstrate the efficiency of different samplers. To prove the convergence guarantees, the following additional assumptions are required. 
\begin{myassumption}
    Assume that the loss functions $\ell_i(\theta)$ is S-smooth, \ie, we have $\Vert \nabla \ell_i(\theta) - \nabla \ell_i({\theta'}) \Vert_2 \leq S \Vert \theta - \theta' \Vert_2$ holds for all $\theta, \theta'$ and $i \in [N]$.
\end{myassumption}

\begin{mylemma} \label{lemma_convergence_for_ce}
    Let $\mathcal{L}(\theta) = \frac{1}{m} \sum_{i=1}^m \ell_i(\theta)$. Assume that the loss function $\ell_i(\theta)$ satisfies: $\Vert \nabla \ell_i(\theta) - \nabla \ell_i(\theta') \Vert_2 \leq S \Vert \theta - \theta' \Vert_2$.
    % \item (Bounded Gradients) $\Vert \mathcal{L}_{CE_i}(\theta)\Vert \leq 2 M $ for all parameters $\theta \in \mathbb{R}^{q}$
    Suppose we run an approximate stochastic gradient descent with stochastic gradient with bounded bias, \ie, $ \Vert \mathbb{E}[g_t | \theta_t] - \nabla \mathcal{L}({\theta_t})\Vert_2 \leq \Delta_t$, and additionally $\Vert g_t \Vert \leq G $ for all $t \in [T]$. Assume that the stepsize is $\eta$, we have that 
    \begin{displaymath}
        \frac{1}{T}\sum_{t=0}^T \mathbb{E}[\Vert \nabla \mathcal{L}(\theta_t) \Vert_2^2] \leq \frac{\mathcal{L}(\theta_0) - \mathcal{L}(\theta^*)}{\eta T} + \frac{1}{2T} \sum_{t=0}^T \Delta_t^2 + \eta SG^2
    \end{displaymath}
\end{mylemma}

\begin{proof}
    From the Lipschitz continuous nature of the function $\mathcal{L}$, we have 
    \begin{displaymath}
    \small
        \begin{split}
        \mathbb{E}\left[\mathcal{L}(\theta_{t+1})\right] & \leq \mathbb{E}\left[\mathcal{L}(\theta_t) + \nabla \mathcal{L} (\theta_t) \cdot (\theta_{t+1} - \theta_t) + \frac{S}{2} \Vert \theta_{t+1} - \theta_t \Vert_2^2\right] \\
         & = \mathbb{E} \left[ \mathcal{L}(\theta_t) - \nabla \mathcal{L}(\theta_t) \cdot \eta \cdot g_t + \frac{\eta^2 S}{2} \Vert g_t \Vert_2^2  \right] \\
         & \leq \mathbb{E} \left[ \mathcal{L}(\theta_t) - \eta \Vert \nabla \mathcal{L}(\theta_t) \Vert_2^2 - \eta \nabla\mathcal{L}(\theta_t) \Delta_t +  \frac{\eta^2 S}{2} \Vert g_t \Vert_2^2 \right] \\
         & \leq \mathbb{E} \left[ \mathcal{L}(\theta_t) - \eta \Vert \nabla \mathcal{L}(\theta_t) \Vert_2^2 - \eta \nabla\mathcal{L}(\theta_t) \Delta_t \right] +  \eta^2 S G^2 \\ 
         & \leq \mathbb{E} \left[ \mathcal{L}(\theta_t) - \eta \Vert \nabla \mathcal{L}(\theta_t) \Vert_2^2 +  \eta \Vert \nabla\mathcal{L}(\theta_t) \Vert_2 \cdot \Delta_t \right] +  \eta^2 S G^2
         \end{split}
         \end{displaymath}
         \begin{displaymath}
         \small
         \begin{split}
         & \leq \mathbb{E} \left[ \mathcal{L}(\theta_t) - \eta \Vert \nabla \mathcal{L}(\theta_t) \Vert_2^2 + \frac{\eta}{2} \left( \Vert \nabla\mathcal{L}(\theta_t) \Vert_2^2 + \Delta_t^2\right) \right] +  \eta^2 S G^2 \\
         & = \mathbb{E} \left[ \mathcal{L}(\theta_t) - \frac{\eta}{2} \Vert \nabla \mathcal{L}(\theta_t) \Vert_2^2 + \frac{\eta}{2} \Delta_t^2 \right] +  \eta^2 S G^2 \\
         & \leq \mathbb{E} \left[ \mathcal{L}(\theta_t) \right]- \frac{\eta}{2} \mathbb{E} \left[ \Vert \nabla \mathcal{L}(\theta_t) \Vert_2^2 \right] + \frac{\eta}{2} \Delta_t^2  +  \eta^2 S G^2
        \end{split}
    \end{displaymath}
    Summing over $t\in [T]$ and using telescoping sum, we have  
    \begin{displaymath}
        \frac{1}{T} \sum_{t=0}^T \mathbb{E}[\Vert \nabla \mathcal{L}(\theta_t) \Vert_2^2] \leq \frac{\mathcal{L}(\theta_0) - \mathcal{L}(\theta_T)}{\eta T} + \frac{1}{2T} \sum_{t=0}^T \Delta_t^2 + \eta S G^2 
    \end{displaymath}
\end{proof}
According to the above lemma, the convergence rate is bounded by the $l_2$ norm of the gradient approximation, and thus we have the following theorems according to the last section for different samplers.
If we take the stepsize $\eta=\frac{\sqrt{\mathcal{L}(\theta_0) - \mathcal{L}(\theta^*)}}{\sqrt{2TS}G}$, the equation follows:
\begin{displaymath}
    \frac{1}{T} \sum_{t=0}^T \mathbb{E}[\Vert \nabla \mathcal{L}(\theta_t) \Vert_2^2] \leq G\sqrt{ \frac{8S \left(\mathcal{L}(\theta_0) - \mathcal{L}(\theta^*)\right)}{T}} + \frac{1}{2T} \Delta_t^2
\end{displaymath}
According to the theoretical analysis, the convergence rate is consistent with $\Delta_t$, which actually corresponds with the upper bound gradient approximation error.
With the easy derivation below, we can use the findings in Section~\ref{sec:theo_gradient_bias}.
\begin{displaymath}
\small
\begin{split}
    \left\Vert \mathbb{E}[g_t | \theta_t] - \nabla \mathcal{L}({\theta_t})\right\Vert_2 & = \left\Vert \frac{1}{\vert \mathcal{B} \vert} \sum_{i=1}^{\vert \mathcal{B} \vert}  \left( \mathbb{E}[\nabla_{\theta_t} \ell'_i] - \nabla_{\theta_t} \ell_i\right) \right\Vert_2\\ 
    & \le \frac{1}{\vert \mathcal{B} \vert }\sum_{i=1}^{\vert \mathcal{B} \vert} \left\Vert   \mathbb{E}[\nabla_{\theta_t} \ell'_i] - \nabla_{\theta_t} \ell_i \right\Vert_2 \le \left\vert   \mathbb{E}[\nabla_{\theta_t} \ell'] - \nabla_{\theta_t} \ell \right\vert = \Delta_t
\end{split}
\end{displaymath}
where $\mathcal{B}$ denotes the data pairs within the mini-batch.

\subsubsection{Uniform Sampler}
\begin{mytheorem}
    Suppose we run the training algorithm with the uniform sampler for $T$ iterations with stepsize $\eta=\frac{\sqrt{\mathcal{L}(\theta_0) - \mathcal{L}(\theta^*)}}{\sqrt{2TS}G}$. We have 
    \begin{displaymath}
        \frac{1}{T} \sum_{t=0}^T \mathbb{E}[\Vert \nabla \mathcal{L}(\theta_t) \Vert_2^2] \leq G\sqrt{ \frac{8S \left(\mathcal{L}(\theta_0) - \mathcal{L}(\theta^*)\right)}{T}} + \frac{U}{2T} \frac{\exp \left(2 \Vert \bm{o} \Vert_\infty \right) - 1 }{M + 1} 
    \end{displaymath}
\end{mytheorem}

\subsubsection{Unigram Sampler}
\begin{mytheorem}
    Suppose we run the training algorithm with the unigram sampler for $T$ iterations with stepsize $\eta=\frac{\sqrt{\mathcal{L}(\theta_0) - \mathcal{L}(\theta^*)}}{\sqrt{2TS}M}$. Under assumption 1 and 2, we have 
    \begin{displaymath}
            \frac{1}{T} \sum_{t=0}^T \mathbb{E}[\Vert \nabla \mathcal{L}(\theta_t) \Vert_2^2] \leq 
            G\sqrt{ \frac{8S \left(\mathcal{L}(\theta_0) - \mathcal{L}(\theta^*)\right)}{T}} + \frac{U}{2T} \frac{\exp \left(2 \Vert \bm{o} \Vert_\infty - \ln N q_{min}\right) - 1 }{M + 1} 
    \end{displaymath}
\end{mytheorem}

\subsubsection{MIDX Sampler}
\begin{mytheorem}
    Suppose we run the training algorithm with the MIDX sampler for $T$ iterations with stepsize $\eta=\frac{\sqrt{\mathcal{L}(\theta_0) - \mathcal{L}(\theta^*)}}{\sqrt{2TS}M}$. Under assumption 1 and 2, we have 
    \begin{displaymath}
        \frac{1}{T} \sum_{t=0}^T \mathbb{E}[\Vert \nabla \mathcal{L}(\theta_t) \Vert_2^2] \leq G\sqrt{ \frac{8S \left(\mathcal{L}(\theta_0) - \mathcal{L}(\theta^*)\right)}{T}} + \frac{U}{2T} \frac{\exp \left(2 \Vert \tilde{\bm{o}} \Vert_\infty \right) - 1 }{M + 1} 
    \end{displaymath}
\end{mytheorem}

\subsection{Generalization Error Bound}\label{sec:generalization_error_bound}
In this section, we derive the generalization error bound for different proposals to estimate the full log-softmax loss. 

Considering \( m \) data points, $(x_1, y_1), \dots,  (x_m, y_m) \stackrel{\text{i.i.d.}}{\sim} \mathcal{D}$, the log-softmax loss follows as 
$$ 
\ell_{\text{softmax}}(x, y)= -\log\left(\frac{\exp(o_y)}{ \sum_{n=1}^N \exp(o_n)}\right),
$$
and the sampled softmax loss follows
$$ 
\ell_{\text{sampled\_softmax}}(x, y, \mathcal{S}_M(x)) = 
-\log\left(\frac{\exp(o_y^\prime)}{ \exp(o_y^\prime) + \sum_{n \in \mathcal{S}_M(x)} \exp(o_n^\prime)}\right).
$$
In this section, we focus on the expected loss function, \ie, the generalization error, and its discrepancy with the actual empirical loss
$$
\hat{\ell}_{empirical} = \frac{1}{m} \sum_{i=1}^m\ell_{\text{sampled\_softmax}}(x_i, y_i, \mathcal{S}_M(x_i)).
$$
Specifically, we aim to estimate the upper bound of the following:
$$
\left| \mathbb{E}_{(x,y)\sim D}\left[\ell_{\text{softmax}}(x,y) \right] - \hat{\ell}_{empirical} \right|.
$$

\begin{mytheorem}
\label{thm:generalization}
Suppose $\|\eqbf{o}\|_{\infty} \leq B_o$. 
% Define the distribution \( P_Q^y(i) = P(i) + P(y)Q(i) \), \( \forall i \in \text{supp}(Q) \).
The following inequality holds with a probability of at least $1 - \delta$:
$$
\begin{aligned}
\mathbb{E}\left[\ell_{\text{softmax}}(x,y)\right] 
& \leq \hat{\ell}_{empirical} + \mathbb{E}\left[ KL(Q||P) \right] + \\
& \inf_{\gamma > 0 }\left(4\gamma + \left( 2 B_o + \log \left(1+\frac{1}{q_{\min }}\right) \right) \sqrt{\frac{\log \mathcal{N}_{\infty}\left(\mathcal{F}, \gamma \right)+\log \frac{2}{\delta}}{2 m}}
\right), 
\end{aligned}
$$
\end{mytheorem}
where \( \mathcal{F} \) represents the score function space:
$
\mathcal{F} := \{o: \mathcal{X} \times \mathcal{Y} \rightarrow \mathbb{R} \ | \ (x, y) \rightarrow o(x, y) \}.
$

\begin{proof}
% [Proof of theorem \ref{thm:generalization}]
We begin with the bias estimation of the sampled softmax. For any given $(x, y)$,
    \begin{equation}
    \small
        \label{eq:sample_softmax_bias}
    \begin{aligned}
    & \ell_{\text{softmax}}(x,y) -  \mathbb{E}_{\mathcal{S}_M}[\ell_{\text{sampled\_softmax}}(x, y, \mathcal{S}_M(x))] = - \log \left(p_y\right) + \mathbb{E}_{\mathcal{S}_M}\left[\log \left(p_y^{\prime}\right)\right] \\
    & = \mathbb{E}_{\mathcal{S}_M}\left[\log \left(\sum_{j=1}^N \exp \left(o_j\right)\right) - \log \left(\exp \left(o_y\right)+\sum_{j \in \mathcal{S}_M} \frac{\exp \left(o_j\right)}{M q_j}\right)\right] \\ 
    & =-\mathbb{E}_{\mathcal{S}_M}\left[\log \left(p_y + \frac{1}{M} \sum_{j \in \mathcal{S}_M} \frac{p_j}{q_j}\right)\right] 
    % \rightarrow-\log \left(p_i+1\right) \text { as } M \rightarrow \infty
    = \mathbb{E}_{\mathcal{S}_M}\left[\log \frac{M} {\sum_{j \in \mathcal{S}_M} \left(\frac{q_j}{p_j + p_y q_j}\right)^{-1}}\right] \\
    & \overset{(a)}{\leq} \mathbb{E}_{\mathcal{S}_M}\left[\log \left(\Pi_{j\in \mathcal{S}_M}\frac{q_j}{p_j + p_y q_j}\right)^{1/M}\right] \\
    & = \frac{1}{M} \sum_{j\in \mathcal{S}_M} \mathbb{E}_{j \sim Q}\left[\log \left(\frac{q_j}{p_j + p_y q_j}\right)\right] \leq \frac{1}{M} \sum_{j\in \mathcal{S}_M} \mathbb{E}_{j \sim Q}\left[\log \left(\frac{q_j}{p_j}\right)\right]   
    % & = KL(Q||P^{y}_{Q}), 
    = KL(Q||P), 
    \end{aligned}
    \end{equation}
    where 
    % we define the distribution \( P_Q^y(i) = P(i) + P(y)Q(i) \), \( i \in \text{supp}(Q) \) and 
    (a) use the 
    Harmonic Mean-Geometric Mean (HM-GM)
    % HM-GM
    inequality.
    % We denote \( o \) as the score function 
    % % replacing the previously used symbol \( o \), 
    % for the label \( y \) given \( x \) within the function space
    % $$
    % \mathcal{F} := \{o: \mathcal{X} \times \mathcal{Y} \rightarrow \mathbb{R} \ | \ (x, y) \rightarrow o(x, y) \}. 
    % $$
    % % where the parameter \(\theta\) is to be optimized.    
    We denote 
    $$ 
    C_{\mathcal{D}}(o) = \mathbb{E}_{\left(x, y \right) \sim \mathcal{D}} 
    \mathbb{E}_{\mathcal{S}_M}[\ell_{\text{sampled\_softmax}}(x, y, \mathcal{S}_M(x))],    
    $$
    and its empirical version
    $$
    {C}_{S}(o) = \frac{1}{m} \sum_{i=1}^m\ell_{\text{sampled\_softmax}}(x_i, y_i, \mathcal{S}_M(x_i)),
    $$
    where $\forall i \in [m], \left(x_i, y_i, \mathcal{J}_{M}(x_i)\right) \in S $.

    Let $\mathcal{G}$ be a $\gamma$-cover of $\mathcal{\mathcal{F}}$  under the infinity norm with the size denoted by $\mathcal{N}_{\infty}\left(\mathcal{F}, \gamma \right)$. 
    Then $\forall \ o \in \mathcal{F}, \exists \ g \in \mathcal{G}$ such that
    $|o(x, y) - g(x, y)| \leq \gamma$.
    We have
    \begin{equation} \small
        \label{eq:gen_cover}
    \begin{aligned}
    & C_{D}(o) - C_{D}(g) \\ 
    & = \mathbb{E}\Big[
    - \log\left(\frac{\exp(o_y)}{ \exp(o_y) + \sum_{j \in \mathcal{S}_M(x)} \exp \left(o_j\right) / (M q_j) }\right)  
    % \\ & \qquad \qquad\qquad 
    + \log\left(\frac{\exp(g_y)}{ \exp(g_y) + \sum_{j \in \mathcal{S}_M(x)} \exp \left(g_j \right) / (M q_j) }\right) \Big] \\
     & = \mathbb{E}\left[ g_y - o_y  
     + 
     \log \left(\frac{ \exp(o_y) + \sum_{j \in \mathcal{S}_M(x)} \exp \left(o_j\right) / (M q_j) }{ \exp(g_y) + \sum_{j \in \mathcal{S}_M(x)} \exp \left(g_j \right) / (M q_j) }\right)\right] \\ 
    & {\ }{\leq} \gamma + \mathbb{E}\left[ \log \left(\frac{ \exp(o_y - g_y) \exp(g_y) + \sum_{j \in \mathcal{S}_M(x)} \exp \left(o_j -g_j \right) \exp(g_j) / (M q_j) }{ \exp(g_y) + \sum_{j \in \mathcal{S}_M(x)} \exp \left(g_j \right) / (M q_j) }\right) \right]  \overset{(a)}{\leq} 2 \gamma,
    \end{aligned}
    \end{equation}
    where (a) holds since $\forall j \in \mathcal{Y}, |o_j - g_j| \leq \gamma $.
    Using the same analysis, we can obtain: 
    \begin{equation}
        \label{eq:emp_cover}
        C_{S}(g) - C_{S}(o) \leq 2 \gamma.
    \end{equation}

Now, we provide an estimate for an upper bound of the loss function.
$$
\begin{aligned}
\ell_{\text{sampled\_softmax}}(x, y, \mathcal{S}_M(x)) 
& = - \log\left(\frac{\exp(o_y)}{ \exp(o_y) + \sum_{j \in \mathcal{S}_M(x)} \exp(o_j) /(M q_j)}\right) \\
& = - o _y + \log\left({ \exp(o_y) + \sum_{j \in \mathcal{S}_M(x)} \exp(o_j) /(M q_j)}\right) \\
& \leq 2 B_o + \log\left(1 + \frac{1}{q_\text{min}}\right).
\end{aligned}
$$
Then, applying the generalization error bound on finite function classes, with at least \(1 - \delta\) probability, we have:
\begin{equation}
    \label{eq:finite_generalize}
    \forall g \in \mathcal{G}, C_{D}(g) - C_{S}(g) \leq  
    \left( 2 B_o + \log \left(1+\frac{1}{q_{\min }}\right) \right) \sqrt{\frac{\log \mathcal{N}_{\infty}\left(\mathcal{F}, \gamma \right)+\log \frac{2}{\delta}}{2 m}}.
\end{equation}

% $$
% \mathcal{F} = \{(x,y, \mathcal{S}_M{(x)}) \rightarrow \ell_{\text{sampled\_softmax}}(x, y, \mathcal{S}_M(x)) \}
% $$

Combining Equations \eqref{eq:gen_cover}, \eqref{eq:emp_cover} and \eqref{eq:finite_generalize},  with at least \(1 - \delta\) probability, we have:
\begin{equation}
    \label{eq:gen_p2}
    \begin{aligned}
    \mathbb{E}[\ell_{\text{sampled\_softmax}}(x, y, \mathcal{S}_M(x))] 
    & \leq \hat{\ell}_{empirical} +  \\
    & 4\gamma + \left( 2 B_o + \log \left(1+\frac{1}{q_{\min }}\right) \right) \sqrt{\frac{\log \mathcal{N}_{\infty}\left(\mathcal{F}, \gamma \right)+\log \frac{2}{\delta}}{2 m}}.
    \end{aligned}
\end{equation}
Combining Equation \eqref{eq:sample_softmax_bias} and taking the infimum over \(\gamma\), we obtain the desired result.
    
\end{proof}
According to the theorem, the generalization error is bounded by the Kullback-Leibler (KL) divergence between the sampling distribution and the ideal softmax distribution. By reducing the bias in the sampling distribution relative to the softmax distribution, we can lower the generalization error and improve model performance. This insight motivates the design of the codeword learning process. The native inverted multi-index relies on product quantization and the K-means clustering algorithm to learn the codewords. Building on this, we aim to improve both the quantization functions and the codeword learning strategy.
Regarding the quantization functions, we adopt residual quantization, which reduces the distortion error and generates smaller residual vectors. This results in a lower bias from the target softmax distribution, as explained in Theorem~\ref{theorem:kl_midx_uni}. For the codeword learning strategy, we propose a learning-based approach by introducing the KL divergence as an objective function, as detailed in Section~\ref{sec:effect_quantizer}.

% cite : [1] Nonasymptotic bounds for suboptimal importance sampling
\section{Experiments} \label{sec:exp}
In this section, we conduct experiments on three typical but different tasks, \ie,  \textit{Language Models}, \textit{Sequential Recommenders} and \textit{Extreme Classification} to validate the utility of the proposed \texttt{MIDX}-based samplers. We will introduce the tasks one by one. 

\subsection{Baseline Samplers}
To sufficiently validate the effectiveness of the proposed \texttt{MIDX}-based samplers, we compare them with both static and adaptive samplers with the log-softmax loss as the objective function. Furthermore, we compare the performance with the native log-softmax loss where all classes contribute to the gradient, with \texttt{Full} as notation.

\noindent \textbf{Static Samplers} refer to those samplers having a fixed sampling distribution over the training process, which is independent of different queries. 
% There are two static samplers to be evaluated.
\begin{itemize}[itemsep=1pt, parsep=2pt]
    \item \textbf{Uniform Sampler}, which selects a class according to the uniform distribution. The sampling probability follows as $p(i)=\frac{1}{N}$, where $N$ denotes the number of classes.
    \item \textbf{Unigram Sampler}, which samples a class according to the frequency, \ie, $p(i)=q(i)= \frac{f_i}{\sum_j f_j}$, where $f_i$ represents the frequency w.r.t. the class $i$ in the training data. 
\end{itemize}

\noindent \textbf{Adaptive Samplers} stand for those samplers being changeable during training and varying for different queries. 
% The sampling distribution should be calculated at different training epochs. 
Although there are various adaptive samplers, such as DNS, the sampling distribution is intractable, unable to fit in the probability-required sampled softmax. Furthermore, we do not adjust the logits with temperatures in the softmax function.
% Thus, we consider the kernel-based samplers here, which have the explicit form of sampling probabilities. 
\begin{itemize}[itemsep=1pt, parsep=2pt]
    \item \textbf{LSH Sampler}~\citep{spring2017new} relies on the collision probability of Local Sensitive Hashing sampling to estimate the partition function in \texttt{log-softmax}. It requires the preprocessing steps to generate the multiple hash tables and then returns the nearest neighbors for the query. By default, there are 16 hash tables and each table has 4 hash functions. 
    \item \textbf{Sphere Sampler}~\citep{blanc2018adaptive} adopts the sphere kernel to approximate the quadratic kernel, whose sampling probability is proportional to $\alpha \cdot s(z,i)^2 + 1$. $s(z, i)$ denotes the logit for the query $\bm{z}$ and the class $i$. The specific value of $\alpha$ is tuned over different data sets, and usually goes greater than 100. Following~\cite{blanc2018adaptive}, the sphere kernel is an alternative for the quadratic kernel, which estimates the softmax of the absolute value, \ie, $\frac{\exp |o_i|}{\sum_{j=1}^N \exp |o_j|}$.
    \item \textbf{Rff Sampler}~\citep{rawat2019sampled} further chooses the random Fourier Features to approximate the softmax, which has a $D$-dimensional RFF map to estimate the kernel. Each class is sampled with the probability $p(i|\bm{z}) \propto \phi(\bm{z})^\top \phi(\bm{q}_i) $, where $$\footnotesize \phi(\bm{x}) = \frac{1}{\sqrt{D}} \left[ \cos(\bm{w}_1^\top \bm{x}),...,\cos (\bm{w}_D^\top \bm{x}), \sin(\bm{w}_1^\top \bm{x}),...,\sin (\bm{w}_D^\top \bm{x})\right].$$ The dimension of the RFF map is always set to 32 by default and the temperature is set to 4. Different from other samplers, the embeddings of query and classes would be normalized before inputting into the function to estimate the Gaussian kernel, \ie, ${\exp\left(\tau \bm{z}^\top \bm{q}_i \right) = \exp \tau \exp (- \tau \Vert \bm{z} - \bm{q}_i \Vert_2^2 / 2 )}$. During our experiments, the performance of the RFF sampler is extremely sensitive to the temperature $\tau$, indicating its less satisfactory for the robustness of the method.
\end{itemize} 
\textbf{Our Samplers} include the dynamic samplers with the inverted multi-index. Concerning the different structures of quantization, we have the following variants:
\begin{itemize}[itemsep=1pt]
    \item \texttt{MIDX-pq} constructs the inverted multi-index depending on the product quantization methods. The overall embedding space is split into two subspaces and the codebooks are built within each subspace.
    \item \texttt{MIDX-rq} builds the index through the residual quantization method and we adopt the two levels for the residual vectors to obtain the final codebooks.
\end{itemize}

\subsection{Language Model Task}
The objective of a language model is to maximize the probability of predicting the target token across the entire vocabulary, given the input sequence of tokens. After encoding the sequence, the resulting representation is treated as the query embedding, while the embeddings corresponding to the vocabulary are treated as class embeddings.

\subsubsection{Experimental Settings}
\textbf{Data set.} There are two data sets to be evaluated. \textbf{Penntreebank}~\citep{marcus-etal-1993-building} data set, a popular benchmark for NLP tasks, consists of English text from Wall Street Journal (WSJ) articles, with a vocabulary size of 10,000. \textbf{Wikitext-2}~\citep{merity2016pointer} data set contains over 30,000 tokens derived from verified Good and Featured articles on Wikipedia, making it approximately twice as large as the Penn Treebank data set. The maximum sequence length is set to 35 for Penn Treebank and 50 for Wikitext-2.

% \vspace{0.5em}
\noindent \textbf{Metric.} Perplexity is used to evaluate the performance of the language models. A lower perplexity score indicates better performance in predicting the next token.

% \vspace{0.5em}
\noindent \textbf{Implementation Details.} We train the models using both LSTM and Transformer architectures. For the LSTM models, the hidden size is set to 128, and the number of layers is 2. For the Transformer models, the number of attention heads is 4, the number of layers is 2, and the feedforward dimension is 1024. The embedding dimension is fixed at 200 for all models. Each approach is trained for 100 epochs, with early stopping based on perplexity performance. The number of sampled classes is fixed at 20 for all evaluated samplers. The default number of codewords within each codebook for \texttt{MIDX-pq} and \texttt{MIDX-rq} is set to 32.

\begin{table}[t]
    \centering
    \begin{tabular}{c|c|c|c|c}
        \toprule
        Data set & \multicolumn{2}{|c|}{Penntreebank} & \multicolumn{2}{|c}{Wikitext-2} \\ \midrule
        Sampler & LSTM & Transformer & LSTM & Tranformer\\ \midrule
        Full          & {109.1965} & 143.8422    & {123.3047} & 180.8331   \\
        Uniform       & {159.9701} & 181.5720    & {211.5420} & 259.4951   \\
        Unigram       & {139.7837} & 166.4322    & {171.6996} & 218.4348   \\
        LSH           & {145.8054} & 167.9671    & {176.8901} & 221.4062   \\
        Sphere        & {143.2146} & 179.2362    & {162.4147} & 273.8121   \\
        RFF           & {145.5703} & 189.1259    & {232.0854} & 278.9223   \\
        MIDX-pq       & {121.5477} & 149.6586    & {136.6786} & 199.7429   \\
        MIDX-rq       & {117.8317} & 147.3405    & {132.2591} & 180.9055   \\
        \bottomrule
    \end{tabular}
    \vspace{-0.8em}
    \caption{Perplexity of Language Model}
    \label{tab:language_model_cp_baselines}
\end{table}
\vspace{-0.5em}

\subsubsection{Performance of Various Samplers}
We begin by evaluating the performance of the language models. In Table~\ref{tab:language_model_cp_baselines}, we compare the perplexity of these models using various samplers. The results demonstrate that the proposed \texttt{MIDX}-samplers, including both the Product-Quantization-based and Residual Quantization-based variants, outperform the other samplers across both data sets.

Among the static samplers, the \textit{Unigram} sampler consistently outperforms the \textit{Uniform} sampler. This suggests that more frequent classes tend to have higher logits during training, which makes the unigram proposal more likely to select informative samples. Among the dynamic samplers, the \textit{Sphere} sampler performs poorly, particularly on the Wikitext-2 data set. This indicates that using an estimation based on the absolute value may not be an effective choice. In contrast, the \texttt{MIDX}-based samplers consistently deliver strong performance, with the Residual Quantization-based variant outperforming the Product Quantization-based variant. This aligns with the smaller distortion error achieved by the residual quantization approach, which contributes to better sampling accuracy.

Furthermore, we evaluate the convergence performance of different samplers by comparing their perplexity changes on the validation data set throughout the training process, as shown in Figure~\ref{fig:exp_change_nlp_samplers}. Within approximately 50 epochs, all samplers converge to low perplexity values close to their final performance. The LSH sampler shows faster convergence but ultimately reaches a higher perplexity compared to the \texttt{MIDX}-samplers. In contrast, the curves for the \texttt{MIDX}-samplers remain consistently closer to the optimal perplexity, indicating more stable and efficient convergence.

\begin{figure}[t]
    \begin{minipage}[t]{0.48\linewidth}
        \centering
        \includegraphics[width=0.95\textwidth]{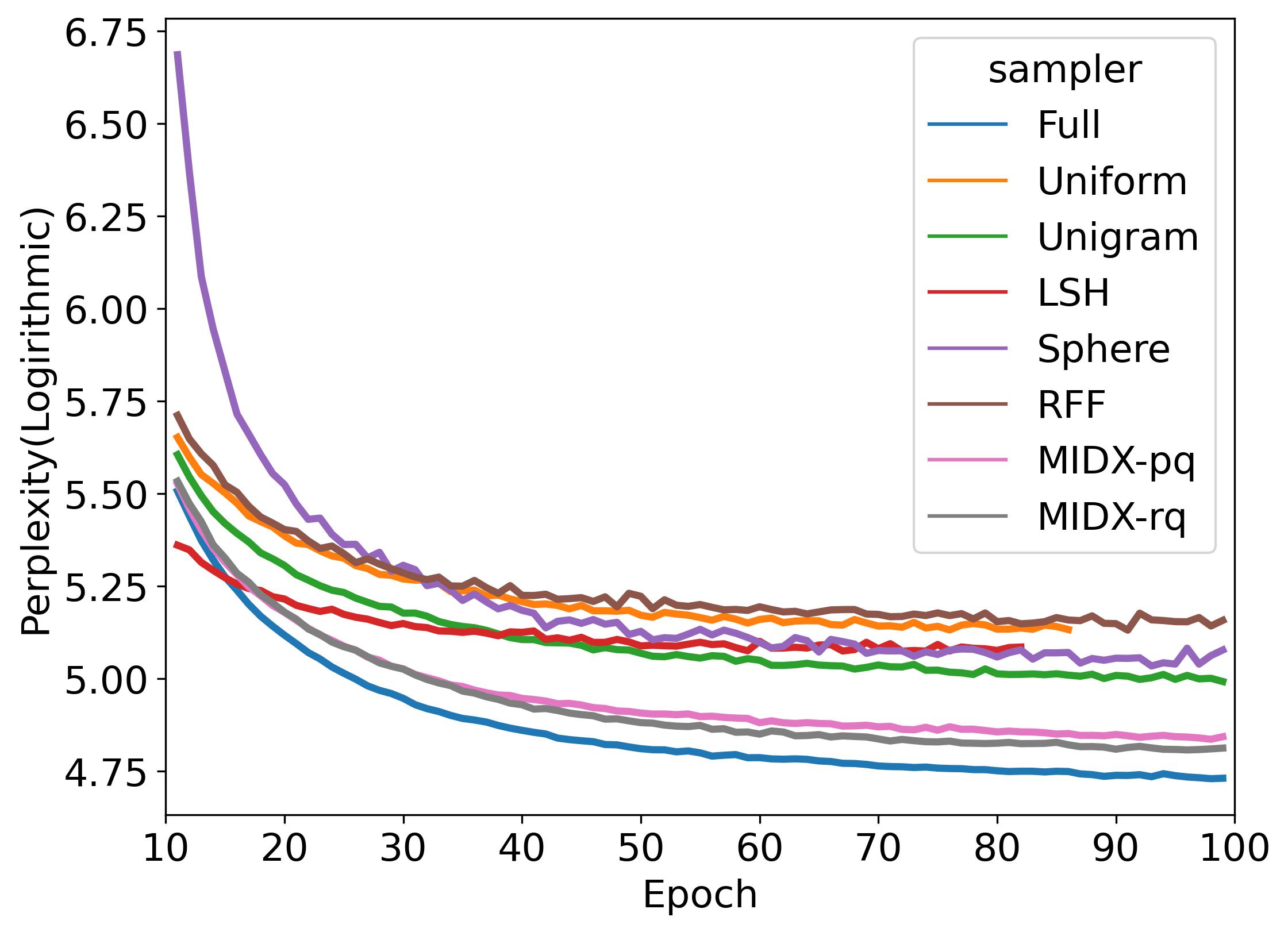}
        \vspace{-0.7em}
        \caption{Comparison with samplers}
        \label{fig:exp_change_nlp_samplers}
    \end{minipage}
    \begin{minipage}[t]{0.48\linewidth}
        \centering
        \includegraphics[width=0.95\textwidth]{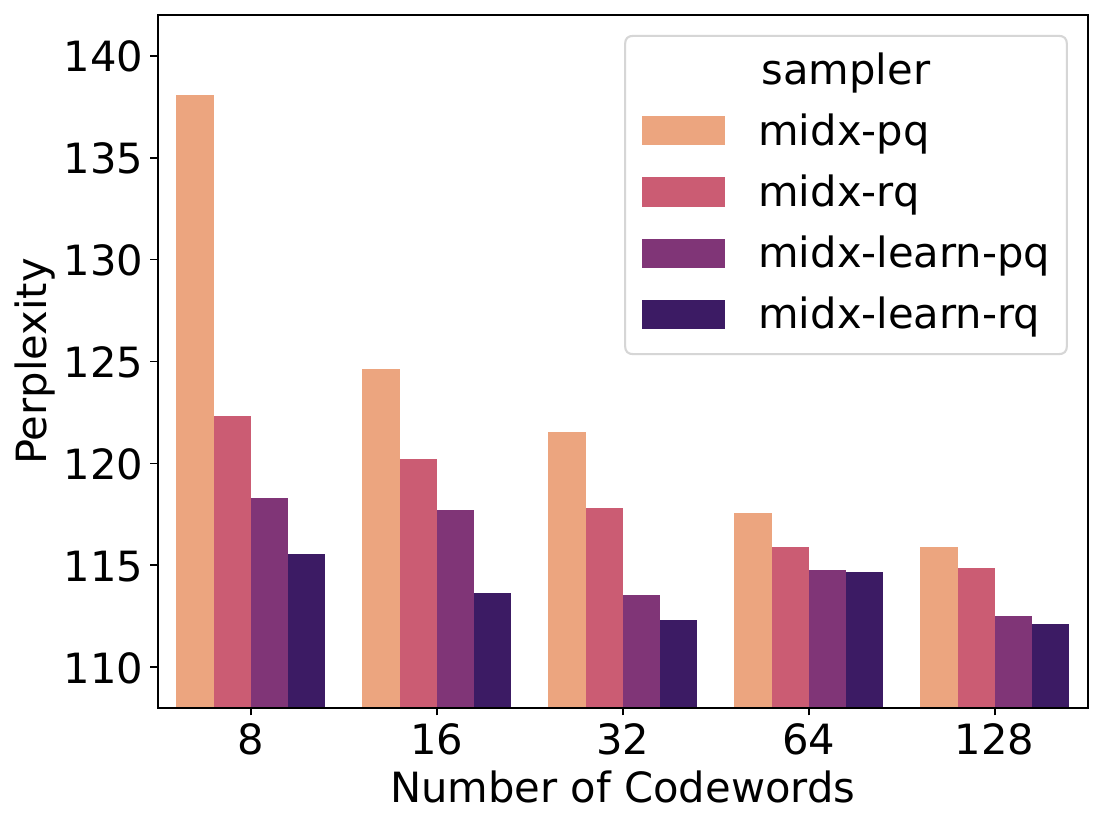}
        \vspace{-0.7em}
        \caption{Effect of codeword numbers}
        \label{fig:penntree_cluster}
    \end{minipage}
\end{figure}

\subsubsection{Effect of Quatization} \label{sec:effect_quantizer}
As suggested by our theoretical analysis, the sampling bias from the softmax distribution is bounded by the norm of the residual vectors, which in turn is determined by the quantizers. This aligns with the distortion error observed in clustering. Table~\ref{tab:language_model_cp_baselines} demonstrates that the residual quantizer outperforms the product quantizer, which is consistent with the fact that the residual quantizer results in a lower distortion error. To gain deeper insights into the effect of different quantizers, we conduct a more detailed analysis in the following discussion.

\noindent \textbf{Effect of the Codewords Numbers.} Firstly, we compare the performance with various numbers of codewords of the codebook. The number of codewords ranges among \{8,16,32,64,128\} for \texttt{MIDX}-based samplers over the \textbf{Penntreebank} data set as shown in Fig~\ref{fig:penntree_cluster}. 
As the number of codewords increases, all \texttt{MIDX}-based samplers demonstrate improved performance. This can be attributed to the fact that the upper bound of the distortion error is inversely proportional to the number of codewords. With more codewords, the residual vector's upper bound in the $\ell_2$ norm decreases, leading to a more accurate sampling distribution that closely approximates the softmax distribution with less bias.

\noindent \textbf{Learnable Codebooks.} 
% The typical inverted multi-index method uses K-Means clustering to build codebooks, which are updated after each training epoch to adapt to changes in the embeddings. This update process involves re-learning the codewords and re-assigning them to different classes. However, a potential issue arises during training: due to the evolving distribution of class embeddings, the most recent step still relies on outdated codebooks for sampling and calculating the sampling proposal. This asynchronous update introduces bias into the estimator and creates inconsistencies in the self-normalized estimator. To address this problem, we propose learnable codebooks, which allow for synchronous updates that align with the latest class distribution.
As we discussed in Section~\ref{sec:generalization_error_bound},  we delved into the pivotal role of minimizing KL divergence from the ideal softmax distribution of samplers in enhancing model generalization ability. This observation drives our endeavor to establish a more consistent learning framework for the codewords. While the conventional inverted multi-index technique employs product quantization with K-Means clustering to refine the codewords based on Euclidean distance optimization, our novel approach emphasizes the automatic acquisition of codewords guided by the KL loss. This method facilitates the automatic updating of codewords, introducing a reconstruction loss as an additional element.

Essentially, we consider the codewords as model parameters, enabling their concurrent optimization with the encoders via gradient descent. The objective function governing the learning of codebooks comprises two key components: the reconstruction loss and the KL-divergence loss. The reconstruction loss follows as:
\begin{displaymath} \small
    \mathcal{L}_{\text{recon}} = \sum_{i=1}^N \left \Vert \hat{\bm{q}}_i - \bm{q}_i  \right\Vert_2^2, \quad \hat{\bm{q}}_i = \left[ \left(\sum_{k=1}^K w_{k}^1 \bm{c}_{k}^1 \right)  \oplus \left(\sum_{k^\prime=1}^K w_{k^\prime}^2 \bm{c}_{k^\prime}^2 \right) \right]
\end{displaymath}
where $\hat{\bm{q}}_i$ represents the encoded embedding. $w_{k}^i$ denotes the weight for the codeword $\bm{c}_k^i$, which is computed by normalizing the inner product between the embedding $\bm{q}_i$ and the codeword $\bm{c}_k^i$:
$w_{k}^i = \frac{\exp (\bm{q}_i^\top \bm{c}_k^i)}{\sum_{k^\prime = 1}^ K \exp (\bm{q}_i^\top \bm{c}_{k^\prime}^i)}$.
The reconstruction loss encourages the encoded embedding $\hat{\bm{q}}_i$ to be closer to the original embedding $\bm{q}_i$.
The KL-divergence loss follows:
\begin{displaymath} \small
    \mathcal{L}_{\text{KL}} = \log \sum_{i=1}^N p_i \cdot \frac{p_i}{ p^{\prime}_i}, \quad p^{\prime}_i = \frac{\exp (\bm{z}^\top \hat{\bm{q}}_i)}{\sum_{j=1}^N \exp (\bm{z}^\top \hat{\bm{q}}_j)}
\end{displaymath}
where $p_i$ denotes the target softmax probability, \ie, $p_i = \frac{\exp (\bm{z}^\top \bm{q}_i)}{\sum_{j=1}^N \exp (\bm{z}^\top \bm{q}_j)}$. The probability $p^{\prime}$  represents the probability derived from the encoded embeddings.
% $$p^{\prime}_i = \frac{\exp (\bm{z}^\top \hat{\bm{q}}_i)}{\sum_{j=1}^N \exp (\bm{z}^\top \hat{\bm{q}}_j)}$$
The KL loss helps reduce the bias in the KL-divergence between the target softmax distribution and the sampling distribution produced by the \texttt{MIDX} sampler. The learnable parameters include the codewords, denoted as $\Theta = \{ \mathcal{C}^1, \mathcal{C}^2\}$, which are updated at each step.

We conducted a comparison on the Penntreebank data set using the LSTM model, with the results presented in Table~\ref{tab:leanable_codewords_penn}. We report the performance in terms of perplexity (PPL) on the test data and the KL-loss on the training data from the final epoch. For better comparison, we also calculate the KL-loss for the original \texttt{MIDX} sampler. The results indicate that by replacing the K-means clustering algorithm with the synchronous learnable codewords, the KL-loss decreases significantly, demonstrating a less biased sampling distribution. This, in turn, leads to a lower perplexity for the language model.

Furthermore, we investigate the effect of varying the number of codewords within the codebooks, as shown in Figure~\ref{fig:penntree_cluster}. Compared to the asynchronous update of codewords, the learnable approach consistently outperforms with a smaller number of codebooks. Specifically, the learnable approach shows an improvement of approximately 14\% in performance with just 8 codewords. This demonstrates the learnable approach’s superior ability to group the classes effectively. With the aid of the KL-loss, the divergence from the softmax distribution is minimized. However, as the number of codewords increases, the performance gap between the two approaches narrows.

\begin{table}[t]
    \centering
    \begin{tabular}{c|c|c|c|c|c}
        \toprule
        Sampler & KL-Loss & PPL    & Sampler       & KL-Loss & PPL    \\ \midrule
        MIDX-pq & 2.2421  & 4.8003 & MIDX-Learn-pq & 1.6625  & 4.7320 \\
        MIDX-rq & 1.5551  & 4.7693 & MIDX-Learn-rq & 1.2905  & 4.7212 \\ 
        \bottomrule
    \end{tabular}
    \vspace{-0.5em}
    \caption{Analysis for Learnable Codebooks on Penntreebank Data set}
    \label{tab:leanable_codewords_penn}
\end{table}

\subsubsection{Sampling Distribution Analysis}
In Section~\ref{sec:theo_distri_bias}, we theoretically compare the different sampling proposals. To provide a more direct comparison between these samplers, we approximate the sampling distribution based on the sampling frequency of each sampler. Specifically, we save the intermediate embeddings for both queries and classes every 10 epochs during the training process. We then apply the different samplers to select a total of $5 \times 10^8$ samples, using the normalized frequency as the probability distribution. For the \textit{softmax} distribution, we directly compute the probabilities. The classes are ordered according to their frequency, and we plot the cumulative probability distribution for a clearer representation.

\begin{figure}[t]
    \begin{minipage}[t]{0.48\linewidth}
        \centering
        \includegraphics[width=0.95\textwidth]{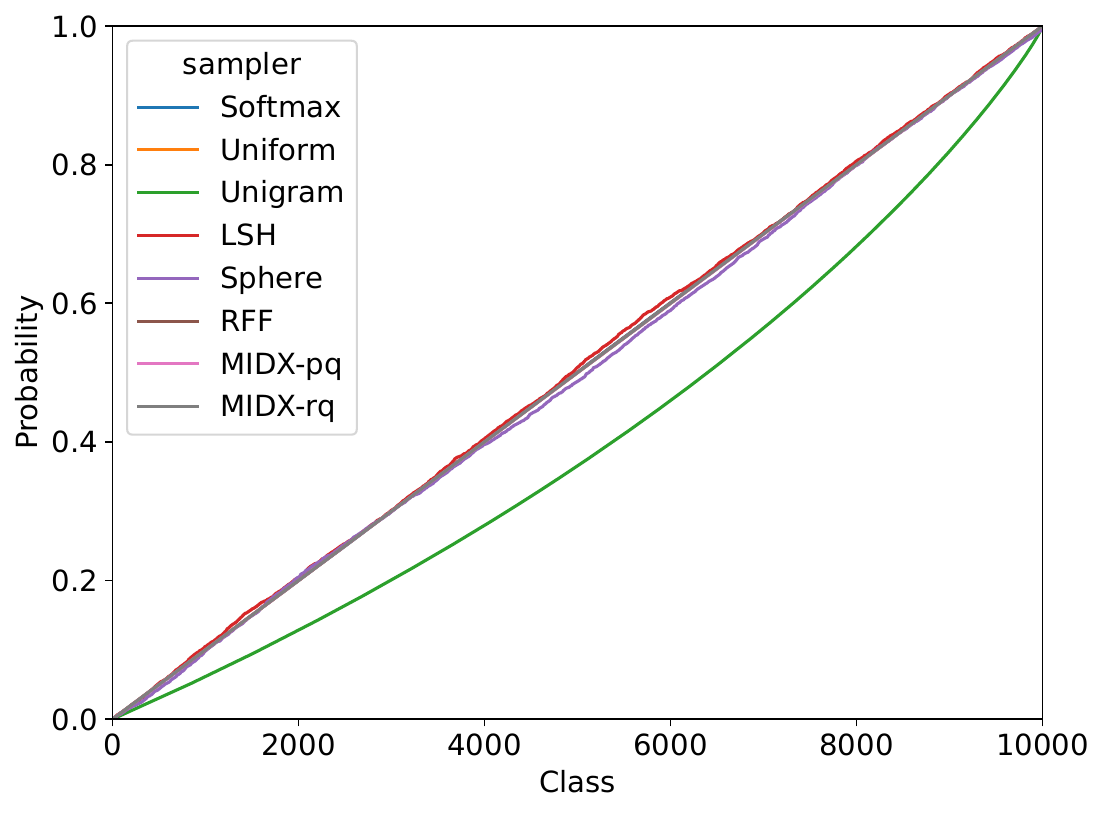}
        \vspace{-0.75em}
        \caption{Sampling probabilities with random initialization.}
        \label{fig:exp_sampling_0_epoch}
    \end{minipage}
    \hspace{0.2cm}
    \begin{minipage}[t]{0.48\linewidth}
        \centering
        \includegraphics[width=0.95\textwidth]{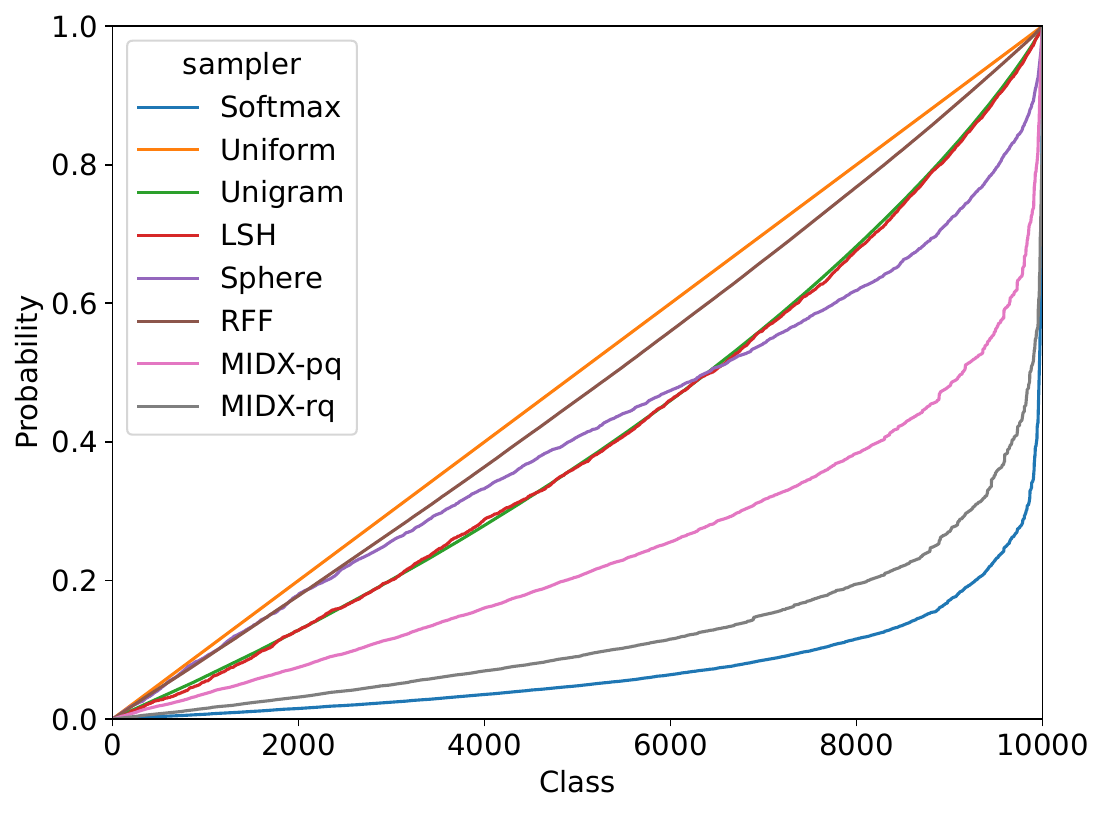}
        \vspace{-0.75em}
        \caption{Sampling probabilities with well-trained embeddings.}
        \label{fig:exp_sampling_40_epoch}
    \end{minipage}
\end{figure}

We first compare the sampling probabilities at the beginning and after training using randomly initialized embeddings (Figure~\ref{fig:exp_sampling_0_epoch}) and well-trained embeddings (Figure~\ref{fig:exp_sampling_40_epoch}). At the start of the training, each embedding is randomly initialized with values drawn from a 0-1 Gaussian distribution, resulting in the softmax probabilities being nearly uniform across all classes. As a result, all adaptive samplers initially behave similarly to the uniform sampler.
However, as training progresses, the softmax probabilities for different classes begin to diverge significantly. The static samplers, Uniform and Unigram, maintain the same distribution as at the beginning. In contrast, the adaptive samplers exhibit different behaviors: the RFF sampler converges towards the Uniform sampler, while the LSH sampler gravitates towards the Unigram sampler. Among the adaptive samplers, the Sphere, \texttt{MIDX}-pq, and \texttt{MIDX}-rq samplers demonstrate better approximations to the softmax distribution, with \texttt{MIDX}-rq showing the closest match to the target distribution.
% Additionally, we track the KL divergence of each sampler across different training epochs, as shown in Figure~\ref{fig:KL_diff_prop}. This reflects the adaptiveness of each sampler, where the \texttt{MIDX} samplers consistently exhibit smaller divergence from the softmax distribution.

\begin{figure}[t]
    \begin{minipage}[t]{0.48\linewidth}
        \centering
        \includegraphics[width=0.98\textwidth]{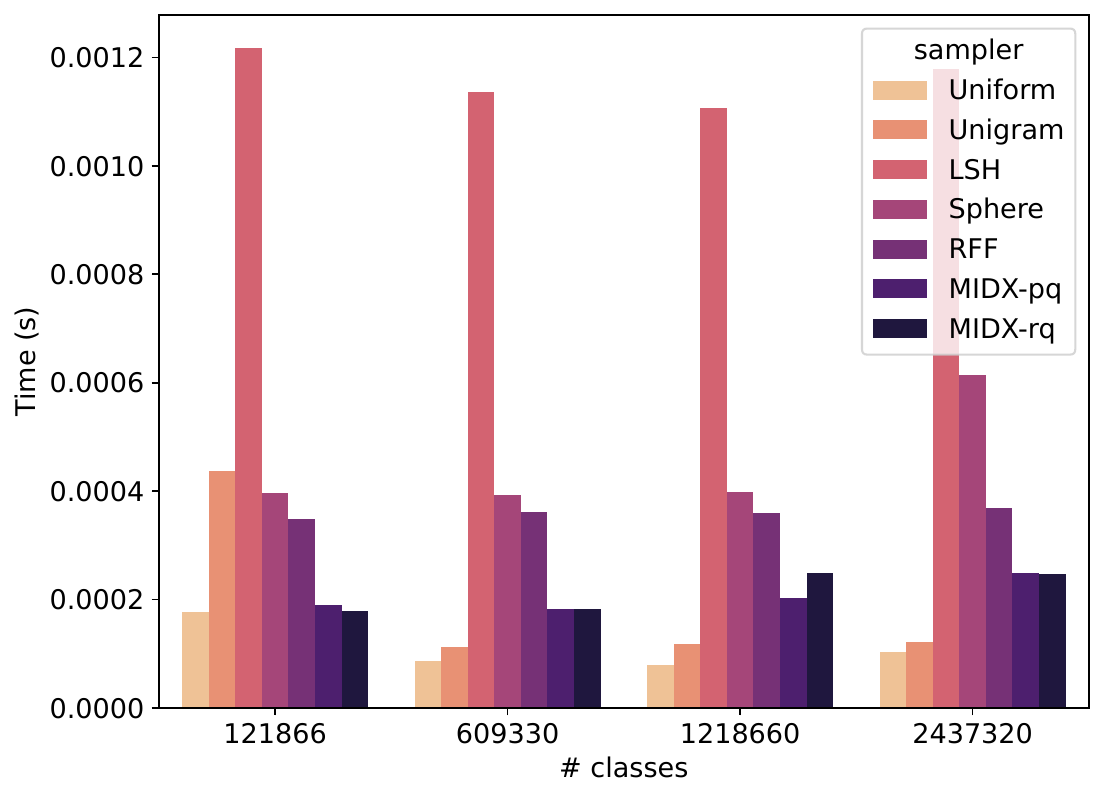}
        \vspace{-0.75em}
    \caption{Sampling Time vs. \#Class}
    \label{fig:exp_samples_sampling_time}
    \end{minipage}
    \begin{minipage}[t]{0.46\linewidth}
        \centering
        \includegraphics[width=0.98\textwidth]{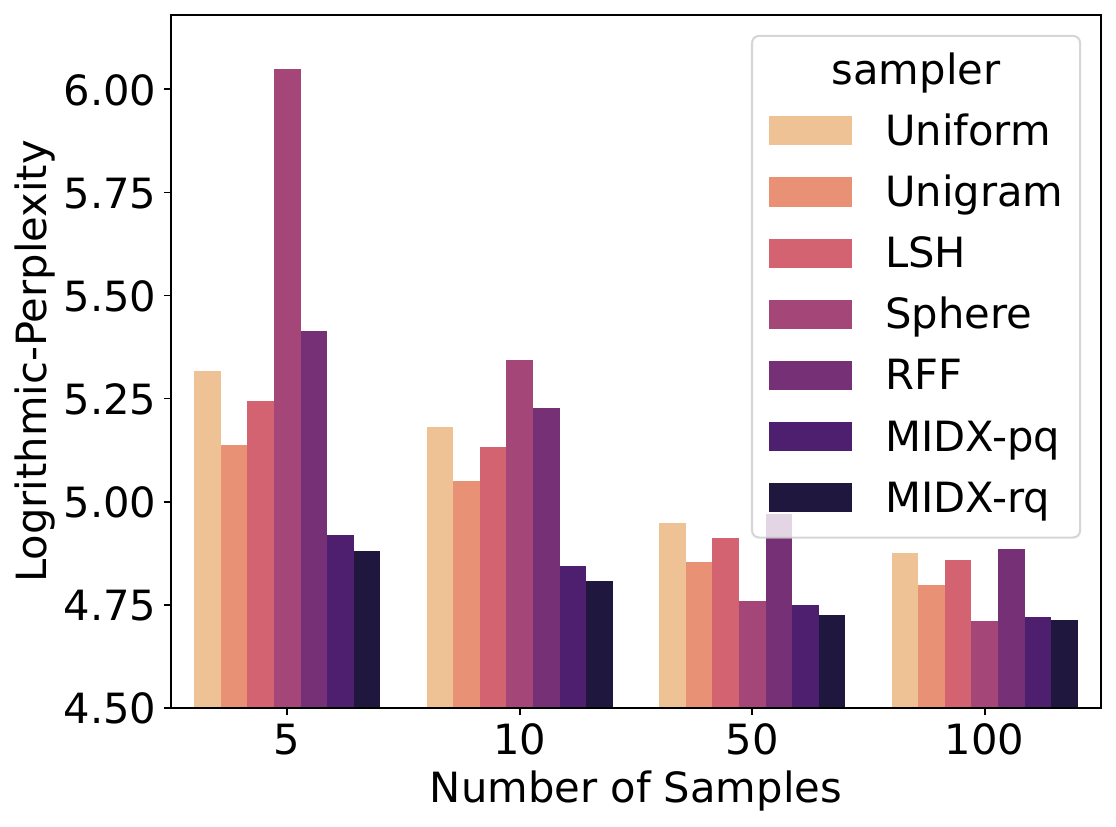}
        \vspace{-0.75em}
        \caption{Effect of \#Samples}
        \label{fig:nlp_num_negs}
    \end{minipage}
\end{figure}

\subsubsection{Effect of Sample Size}
Considering that the gradient bias is inversely proportional to the size of the sampling pool, we conduct experiments to explore the effect of sample size on the Penntreebank data set, as illustrated in Figure~\ref{fig:nlp_num_negs}. We vary the number of samples in the set ${5, 10, 50, 100}$. As expected, all samplers perform better with an increased number of samples in terms of perplexity. Even with a limited number of samples, the \texttt{MIDX}-samplers still demonstrate competitive performance, achieving a log-perplexity smaller than 5 and outperforming all other samplers. An interesting observation regarding the Sphere sampler is that its performance improves significantly as the number of samples increases, which suggests that this method for approximating the absolute value of softmax has a substantial bias.

% can have some similar conclusions from ESS. effective sample size (Adaptive importance sampling to accelerate training of a neural probabilistic language model)

\subsubsection{Sampling Time}
We report the sampling time for various samplers as the number of classes increases, as shown in Figure~\ref{fig:exp_samples_sampling_time}. Note that this analysis does not include initialization time, such as the frequency preparation for the Unigram sampler and the K-means clustering for \texttt{MIDX}-based samplers. We sample 100 classes for a batch of 256 users and calculate the average time on GPUs over repeated 10 trials.

The results show that the static samplers, including the Uniform and Unigram samplers, maintain consistent sampling times as the number of classes increases, demonstrating superior sampling efficiency. In contrast, among the adaptive samplers, the LSH sampler incurs the highest time cost, while other samplers perform more efficiently. For a large number of classes, \texttt{MIDX}-based samplers continue to perform well, maintaining nearly the same sampling time even with approximately 100,000 classes. This is expected, as the time complexity of \texttt{MIDX}-based samplers scales linearly with the number of codewords (set to 64 in our experiments), rather than the number of classes. In contrast, the kernel-based samplers show an increasing time cost as the number of classes grows. This performance degradation may be attributed to the specific GPU implementation we employed, which does not uses tree structures as described in the original approach. As a result, the time complexity of the kernel-based samplers is more sensitive to the number of classes.

\subsection{Sequential Recommendation}
% Sequential recommenders aim to predict a user's preferred items based on their behavioral history. The query corresponds to the user embedding, which is derived from encoding the user's behavior history, while the classes refer to the candidate items.

\subsubsection{Experimental Settings}
\noindent \textbf{Data set.} Three ground-truth data sets are used for the experiments, as summarized in Table~\ref{tab:data_seqrec}. \textbf{MovieLens-10M} contains approximately 10,000 user ratings for various movies, with ratings above 3 retained in the data set. \textbf{Gowalla} records user check-ins, while \textbf{Amazon-books} includes user ratings for books, with ratings greater than 3 kept in the data set. Users and items with fewer than 10 interactions are filtered out. User sequences are constructed based on the timestamp of each record, with a maximum sequence length of 20. The data is split into training, validation, and test sets with a ratio of 8:1:1.

\begin{table}[ht]
    \centering
    \begin{tabular}{c|r|r|r|r}
        \toprule
        Data set & \# User & \# Item & \# Interaction & Density\\
        \midrule
        MovieLens-10M & 69,585  & 9,176 &  8,233,567 & 0.0129 \\
        Gowalla & 52,986 & 121,867 & 3,301,571 & 0.0005\\
        Amazon-books & 68,498 & 65,549 & 2,954,716 & 0.0007\\
        \bottomrule
    \end{tabular}
    \vspace{-0.5em}
    \caption{Data Statistics for sequential recommendation}
    \label{tab:data_seqrec}
\end{table}

\noindent \textbf{Metric.} Two metrics, NDCG@k and Recall@k, are used to evaluate the performance of different samplers. Here, $k$ denotes the cutoff for the predicted ranking list, with values set to ${10, 20, 50}$. NDCG (Normalized Discounted Cumulative Gain) assigns higher scores to items ranked higher in the list, while Recall measures the ratio of relevant items that appear in the top-$k$ rankings. In both metrics, a higher value indicates better performance. Although these metrics do not directly assess classification tasks, previous research has shown that ranking-oriented metrics are strongly correlated with softmax loss~\citep{menon2019multilabel}.

\noindent \textbf{Implementation Details.} 
The model is trained using the Adam optimizer, with both model parameters and hyperparameters (\eg, learning rate, weight decay) tuned on the validation set. Early stopping is applied, with the stopping criterion based on NDCG@10, using a patience of 10 epochs. The maximum number of epochs is set to 50. The embedding dimension is set to 64. For batch size, we use 4096 for the MovieLens-10M and Gowalla data sets, and 2048 for the Amazon-books data set. The number of sampled items for optimizing the softmax loss is set to 90.

We evaluate two typical sequential recommendation models to encode user behavior: SASRec~\citep{kang2018self} and GRU4Rec~\citep{jannach2017recurrent}. For SASRec, we use 2 attention layers, each with 2 heads and 128 hidden units. For GRU4Rec, the hidden size is set to 128.

\subsubsection{Experimental Results With Various Samplers}

\newcommand{\tabincell}[2]{\begin{tabular}{@{}#1@{}}#2\end{tabular}} 
\begin{table}[t]
\begin{tabular}{c|c|cccc|cccc}
\toprule
                         &         & \multicolumn{4}{|c|}{SASRec}        & \multicolumn{4}{|c}{GRURec}        \\ \midrule
Data set                 & Sampler & N@10   & N@50   & R@10   & R@50   & N@10   & N@50   & R@10   & R@50   \\ \midrule
\multirow{8}{*}{ML-10m}  & Full    & 0.0922 & 0.1440 & 0.1738 & 0.4114 & 0.1358 & 0.1892 & 0.2365 & 0.4808 \\
                         & Uniform & 0.0840 & 0.1371 & 0.1623 & 0.4058 & 0.1224 & 0.1797 & 0.2270 & 0.4882 \\
                         & Unigram & 0.0885 & 0.1406 & 0.1705 & 0.4100 & 0.1261 & 0.1818 & 0.2304 & 0.4847 \\
                         & LSH     & 0.0822 & 0.1338 & 0.1601 & 0.3977 & 0.1142 & 0.1684 & 0.2120 & 0.4602 \\
                         & Sphere  & 0.0916 & 0.1431 & 0.1744 & 0.4110 & 0.1334 & 0.1870 & 0.2374 & 0.4819 \\
                         & RFF     & 0.0871 & 0.1400 & 0.1684 & 0.4108 & 0.1224 & 0.1791 & 0.2275 & 0.4864 \\
                         & MIDX-pq & 0.0899 & 0.1419 & 0.1721 & 0.4102 & 0.1336 & 0.1873 & 0.2355 & 0.4802 \\
                         & MIDX-rq & 0.0916 & 0.1433 & 0.1752 & 0.4125 & 0.1337 & 0.1877 & 0.2355 & 0.4817 \\ \midrule \midrule
\multirow{8}{*}{Gowalla} & Full    & 0.0347 & 0.0514 & 0.0624 & 0.1393 & 0.0318 & 0.047  & 0.0571 & 0.1273 \\
                         & Uniform & 0.0265 & 0.0416 & 0.0483 & 0.1176 & 0.0236 & 0.0387 & 0.0441 & 0.1138 \\
                         & Unigram & 0.0271 & 0.0421 & 0.0495 & 0.1190 & 0.0234 & 0.0385 & 0.0444 & 0.1140 \\
                         & LSH     & 0.0284 & 0.0442 & 0.0524 & 0.1252 & 0.0250 & 0.0407 & 0.0470 & 0.1195 \\
                         & Sphere  & 0.0312 & 0.0476 & 0.0561 & 0.1315 & 0.0281 & 0.0435 & 0.0519 & 0.1226 \\
                         & RFF     & 0.0277 & 0.0429 & 0.0507 & 0.1206 & 0.0241 & 0.0389 & 0.0454 & 0.1144 \\
                         & MIDX-pq & 0.0337 & 0.0500 & 0.0605 & 0.1356 & 0.0293 & 0.0447 & 0.0534 & 0.1245 \\
                         & MIDX-rq & 0.0332 & 0.0495 & 0.0596 & 0.1350 & 0.0308 & 0.0463 & 0.0563 & 0.1281 \\ \midrule \midrule
\multirow{8}{*}{Amazon}  & Full    & 0.0644 & 0.0881 & 0.1035 & 0.2126 & 0.0608 & 0.0849 & 0.0995 & 0.2108 \\
                         & Uniform & 0.0467 & 0.0700 & 0.0819 & 0.1898 & 0.0379 & 0.0610 & 0.0699 & 0.1769 \\
                         & Unigram & 0.0473 & 0.0704 & 0.0824 & 0.1897 & 0.0401 & 0.0629 & 0.0738 & 0.1798 \\
                         & LSH     & 0.0497 & 0.0736 & 0.0857 & 0.1963 & 0.0440 & 0.0670 & 0.0793 & 0.1859 \\
                         & Sphere  & 0.0549 & 0.0788 & 0.0934 & 0.2038 & 0.0549 & 0.0794 & 0.0949 & 0.2081 \\
                         & RFF     & 0.0471 & 0.0700 & 0.0828 & 0.1889 & 0.0386 & 0.0616 & 0.0714 & 0.1779 \\
                         & MIDX-pq & 0.0608 & 0.0846 & 0.0998 & 0.2096 & 0.0530 & 0.0771 & 0.0924 & 0.2041 \\
                         & MIDX-rq & 0.0622 & 0.0863 & 0.1020 & 0.2134 & 0.0547 & 0.0791 & 0.0948 & 0.2075 \\ \bottomrule
\end{tabular}
\vspace{-0.75em}
\caption{Comparison of Sequential Recommenders. (N: NDCG and R: Recall)}
    \label{tab:seqrec_cp_baselines}
\end{table}

The experimental results in terms of NDCG and Recall at cutoffs 10 and 50 are reported in Tables~\ref{tab:seqrec_cp_baselines}. Overall, \texttt{MIDX} samplers consistently outperform other samplers across all data sets and metrics, especially at higher cutoffs of 50. 
% LSH performs the worst among the adaptive samplers, while Sphere shows competitive performance. 
Our key findings are as follows:

\textbf{Finding 1: Static samplers, Uniform and Unigram samplers, consistently lag behind dynamic samplers.} On all data sets, Uniform and Unigram samplers consistently show poorer performance, especially at higher cutoffs (N@50 and R@50). For example, on MovieLens-10M, Unigram achieves 0.1406 for N@50 and 0.4100 for R@50, which are significantly lower than \texttt{MIDX}-based samplers.

\textbf{Finding 2: The sparse Gowalla data set benefits most from \texttt{MIDX} samplers.} Gowalla (with a density of 0.0005) shows the largest improvement with \texttt{MIDX} samplers compared to Uniform. \texttt{MIDX}-rq achieves a 14.8\% improvement in Recall@50 and 22.3\% improvement in NDCG@50 over Uniform. The sparse nature of Gowalla means data points are rare and dispersed, making it difficult for static samplers like Uniform to capture meaningful patterns. In contrast, \texttt{MIDX} samplers reduce sampling bias, especially in regions with sparse interactions, leading to better recommendations, which is important for sparse scenarios.

\textbf{Finding 3: \texttt{MIDX}-rq (Residual Quantization) consistently outperforms \texttt{MIDX}-pq (Product Quantization), } demonstrating the advantage of residual quantization in reducing distortion errors. Residual Quantization encodes residuals (differences between the target distribution and the quantized representation) more accurately, leading to better compression and improved softmax approximation. This results in better ranking and recommendation performance for sequential tasks.

\subsection{Extreme Classification Task}

\subsubsection{Experimental Settings}
We evaluate the samplers on two widely used public data sets: AmazonCat-13K and WikiLSHTC-325K. The statistics of these data sets are summarized in Table~\ref{tab:data_ec}. Each data point is represented as a sparse Bag-of-Words (BOW) feature vector, with each having at least one associated label. Following prior work~\citep{rawat2019sampled}, we represent each data entry as a 128-dimensional vector, $\bm{z} \in \mathbb{R}^{d}$, where $d$ is set to 128. Similarly, each class label is mapped to a 128-dimensional vector $\bm{q} \in \mathbb{R}^{d}$.
We use the original train-test splits provided by the data sets and evaluate performance using Precision@$k$ for $k = 1, 3, 5$. For each data point, we sample 1000 classes across all three data sets, using the specified sampler, to compute the log-sampled softmax. All models are trained for 50 epochs.

% \noindent \textbf{Metric} 
\begin{table}[t]
    \centering
    \begin{tabular}{c|r|r|r}
        \toprule
        Data set & \# Class & \# Samples & Dimensionality\\
        \midrule
        AmazonCat-13K & 13,330 & 1,186,239& 203,882\\
        % Delicious-200K & 205,443& 196,606& 782,585\\
        WikiLSHTC-325K & 325,056& 1,778,351& 1,617,899\\
        \bottomrule
    \end{tabular}
    \vspace{-0.7em}
    \caption{Data Statistics for Extreme Classification}
    \label{tab:data_ec}
\end{table}

\subsubsection{Experimental Results With Various Samplers}
To verify the performance of different samplers, we compare the different samplers on the different data sets, as illustrated in Table~\ref{tab:ec_cp_baselines}. Both \texttt{MIDX}-pq and \texttt{MIDX}-rq demonstrate strong performance across the two datasets. Among the other baseline samplers, the Sphere sampler exhibits comparable results, particularly on the WikiLSHTC dataset. However, for larger cutoffs, such as precision@3 and precision@5, \texttt{MIDX}-based samplers outperform the others, highlighting their advantage in predicting overall probabilities. In extreme classification tasks, where class embedding dimensionality is significantly large, the consistent performance of \texttt{MIDX}-based samplers indicates their effectiveness even in high-dimensional settings, which are commonly encountered in deep learning models.

% \begin{table}
%     \centering
%     \caption{Comparison of classification performance}
%     \begin{tabular}{c|c|ccccc|c}
%         \toprule
%         Data set       & Metric & Full   & Uniform & Unigram & Sphere & RFF    & MIDX-uni\\
%         \midrule
% \multirow{3}{*}{AmazonCat} & P@1 & 0.8509 & 0.7399  & 0.8105  & 0.7994 & 0.8112 & 0.8380  \\
%               & P@3 & 0.7238 & 0.6373  & 0.6819  & 0.6835 & 0.6955 & 0.7091  \\
%               & P@5 & 0.5843 & 0.5235  & 0.5502  & 0.5568 & 0.5653 & 0.5691\\
%         \midrule
%         \midrule
% \multirow{3}{*}{Delicious}     & P@1 & 0.4334 & 0.2186  & 0.3745  & 0.4203 &        & 0.4336\\
%               & P@3 & 0.3832 & 0.2600    & 0.3459  & 0.3742 &        & 0.3817\\
%               & P@5 & 0.3558 & 0.2595  & 0.3271  & 0.3487 &        & 0.3542\\
%         \midrule
%         \midrule
% \multirow{3}{*}{WikiLSHTC}     & P@1 & 0.4930  & 0.1850   & 0.2868  & 0.4761 &        & 0.3360 \\
%               & P@3 & 0.3088 & 0.1376  & 0.1939  & 0.3011 &        & 0.2255 \\
%               & P@5 & 0.2290  & 0.1137  & 0.1532  & 0.2249 &        & 0.1747   \\
%         \bottomrule
%     \end{tabular}
%     \label{tab:ec_cp_baselines}
% \end{table}

\begin{table}[t]
    \centering
    \begin{tabular}{c|c|c|c|c|c|c|c|c}
    \toprule
    Data set                    & Sampler & P@1    & P@3    & P@5    & Data set                    & P@1    & P@3    & P@5    \\
    \midrule
    \multirow{8}{*}{AmazonCat} & Full    & 0.8478 & 0.7169 & 0.5770 & \multirow{8}{*}{WikiLSHTC} & 0.1805 & 0.0867 & 0.0596 \\
                               & Uniform & 0.7242 & 0.6284 & 0.5152 &                            & 0.1006 & 0.0495 & 0.0356 \\
                               & Unigram & 0.8105 & 0.6819 & 0.5502 &                            & 0.1504 & 0.0676 & 0.0457 \\
                               & LSH     & 0.7936 & 0.6704 & 0.5405 &                            & 0.1462 & 0.0659 & 0.0447 \\
                               & Sphere  & 0.8176  & 0.6950 & 0.5602 &                            & 0.1662 & 0.0744 & 0.0501 \\
                               & RFF     & 0.7484 & 0.6441 & 0.5285 &                            & 0.1455 & 0.0652 & 0.0445 \\
                               & MIDX-pq & 0.8352 & 0.7055 & 0.5652 &                            & 0.1661 & 0.0779 & 0.0531 \\
                               & MIDX-rq & 0.8478 & 0.7166 & 0.5739 &                            & 0.1593 & 0.0758 & 0.0518 \\ \bottomrule
    \end{tabular}
    \vspace{-0.7em}
    \caption{Comparison of classification performance}
    \label{tab:ec_cp_baselines}
\end{table}

\subsection{Overall Experimental Analysis}
The \texttt{MIDX} samplers consistently outperform other methods across all tasks, demonstrating superior performance in various scenarios. When compared to other samplers, \texttt{MIDX} exhibits several notable advantages:

\textbf{Kernel-based Samplers (Sphere and RFF)}:
The Sphere sampler performs well in sequential and extreme classification tasks but struggles with language modeling, indicating its less unsatisfactory in generalization. Similarly, the RFF sampler, although competitive in some cases, often delivers results comparable to the Unigram sampler. Both of these samplers rely on kernel-based estimators, which have inherent limitations in capturing the true softmax distribution. While kernel methods are proportional to the logits, they are unable to fully align with the actual sampling distribution. This results in a discrepancy between the computed probabilities and the true distribution, which can cause issues when estimating gradients. In contrast, \texttt{MIDX} directly estimates the softmax distribution by considering both the numerator and denominator, which allows for more accurate approximations and better integration with self-normalized importance sampling.

\textbf{LSH (Locality Sensitive Hashing)}:
LSH underperforms relative to dynamic samplers and is highly sensitive to hyperparameters, such as the number of tables and hash functions. These parameters significantly influence the allocation of data points and the sampler’s overall effectiveness. On the other hand, \texttt{MIDX} is more efficient, requiring only a single hyperparameter—the number of codewords in each book. This simplicity allows \texttt{MIDX} to perform effectively, even with a small number of codewords.

\textbf{Static Samplers (Unigram and Uniform)}:
The Unigram sampler consistently outperforms the Uniform sampler, but both are static methods that cannot adapt to changes during model training. This inability to adapt to evolving patterns in the data limits their performance, especially in dynamic environments. In contrast, \texttt{MIDX}, being a dynamic sampler, is more capable of capturing changes in the data distribution and adapting its sampling strategy accordingly, leading to improved performance across tasks.

\section{Conclusion}
In this paper, we introduced the \texttt{MIDX}-Sampler, a novel solution to the computational challenges of the softmax function in large-scale multi-class classification tasks. By exploiting the inverted multi-index structure, the softmax probability can be decomposed into multiple multinomial distributions, improving approximation accuracy. We further optimize sampling efficiency by replacing the query-aware residual probability with a uniform distribution, achieving faster sampling without compromising performance. From a theoretical perspective, we conduct a comprehensive analysis, addressing key aspects such as the KL divergence between the sampling proposals and the target softmax distribution, gradient bias compared to the full softmax, convergence rate, and generalization error bounds. The theoretical results show that reducing the bias in the sampling distribution relative to the softmax leads to faster convergence and improved generalization performance. This insight motivates the design of quantizers to minimize the KL divergence to the softmax distribution.
Finally, we present extensive experiments on large-scale language models, sequential recommendation systems, and extreme multi-class classification tasks. The results demonstrate that the MIDX-Sampler consistently outperforms existing methods in both efficiency and effectiveness. Our work offers a powerful alternative to softmax-based approaches, providing a scalable solution for large-scale multi-class problems and paving the way for future advancements in efficient machine learning applications.

% Acknowledgements should go at the end, before appendices and references

\acks{The work was supported by grants from the National Natural Science Foundation of China (No. U24A20253).}


\newpage

\appendix
\section*{Appendix A. Summary of Notations}
We summarize the notations within this paper in the following table for reference:
\begin{table}[ht]
    \centering
    \begin{tabular}{c|c}
    \toprule
    Notation     & Description  \\ \midrule
      $N$   & Number of Classes, $\vert \mathcal{Y} \vert$ \\
      $M$   & Number of Sampled Negative Classes \\
      $\mathcal{Y}$ & Class Set \\
      $\mathbb{S}_M$ & Set of sampled classes \\
      $\bm{z}$ & $D$-dimensional Query Embedding after encoder\\ 
      $ \bm{q}_i $ & $D$-dimensional Class Embedding after encoder \\ 
      $\tilde{\bm{q}}$ & $D$-dimensional Residual Vector after quantizer\\ 
      $o_i$ & $\bm{z}^\top \bm{q}_i$, Logit for the class $i$ \\
      $ \tilde{o}_i $ & $\bm{z}^\top \tilde{\bm{q}}_i$, Residual score for the class $i$ \\
      $p_i$ & Softmax probability \\
      $q_i$ & Sampling probability \\
      $D$ & Embedding Dimension \\
      $ \mathcal{C}_l $ & $l$-th Codebook \\ 
      $ \bm{c}_j^l$ & The $j$-th codeword in the codebook $ \mathcal{C}_l $ \\
      $B$ & Number of Codebooks\\
      $K$ & Number of codewords in each Codebook \\
    \bottomrule
    \end{tabular}
    \caption{Notations}
    \label{tab:notations}
\end{table}

\section*{Appendix B. Proof Details}
% \label{app:theorem}

\subsection*{Appendix B.1 Proof in Section~\ref{sec:theo_distri_bias}}
\noindent
{\bf Theorem~\ref{theorem:kl_unif}} {\it Assume $o_i = \bm{z}^\top \bm{q}_i$ denotes the similarity score with respect to the query $\bm{z}$ and the $i$-th class and $\bm{o}\in \mathbb{R}^{N}$ is a vector involving all similarity scores over all classes. The KL-divergence from the target softmax distribution $P(\cdot|\bm{z})$ to the proposal distribution $Q_{\text{uniform}} (\cdot | \bm{z})$ according to the uniform sampling can be bounded from below:
    \begin{displaymath}
        0 \le \mathcal{D}_{KL}\left[   Q_{\text{uniform}} (\cdot | \bm{z}) \Vert P  (\cdot | \bm{z}) \right] \leq 2 \Vert \bm{o} \Vert_{\infty}.
    \end{displaymath}
}

\begin{proof}
Given the softmax distribution 
$$ P (i | \bm{z}) = \frac{\exp (o_i)}{\sum_{j=1}^N \exp (o_j)} $$ and the sampling distribution of Uniform sampler $$ Q_{\text{uniform}} (i | \bm{z}) = \frac{1}{N} $$ the KL-divergence has the following bound:
    \begin{displaymath}
        \begin{split}
        \mathcal{D}_{KL}\left[   Q_{\text{uniform}} (\cdot | \bm{z}) \Vert P  (\cdot | \bm{z}) \right] & = \sum_{i=1}^N Q_{\text{uniform}} (i | \bm{z}) \cdot \log \frac{Q_{\text{uniform}} (i | \bm{z})}{P (i | \bm{z})} \\ 
        & = \sum_{i=1}^N \frac{1}{N} \cdot \log \frac{\sum_{j=1}^N \exp (o_j)}{N \cdot \exp (o_i)} \\
        & \le \sum_{i=1}^N \frac{1}{N} \cdot \log \frac{\sum_{j=1}^N  \exp ( \Vert \bm{o} \Vert_{\infty})}{N \exp( - \Vert \bm{o} \Vert_{\infty})} \\
        & = \sum_{i=1}^N \frac{1}{N} \cdot \log \exp (2\Vert \bm{o} \Vert_{\infty}) =  2\Vert \bm{o} \Vert_{\infty}\\ 
    \end{split}
    \end{displaymath}
    $\mathcal{D}_{KL}\left[  Q_{\text{uniform}} (\cdot | \bm{z})  \Vert P (\cdot | \bm{z})  \right]  > 0 $ holds due to the non-negativity of the KL-divergence.
\end{proof}

\noindent
{\bf Theorem~\ref{theorem:kl_unig}} {\it
Assume $o_i = \bm{z}^\top \bm{q}_i$ denotes the similarity score with respect to the query $\bm{z}$ and the $i$-th class and $\bm{o}\in \mathbb{R}^{N}$ is a vector involving all similarity scores over all classes. $q_{min}$ and $q_{max}$ denote the minimal and maximal probability according to the normalized unigram distribution given the data distribution.
The KL-divergence from the softmax distribution $P(\cdot|\bm{z})$ to the proposal distribution $Q_{\text{unigram}} (\cdot| \bm{z})$ according to the unigram sampling can be bounded from below:
    \begin{displaymath}
        0 \leq \mathcal{D}_{KL}\left[   Q_{\text{unigram}} (\cdot | \bm{z}) \Vert P  (\cdot | \bm{z}) \right] \leq 2 \Vert \bm{o} \Vert_{\infty} + \ln N q_{max}
    \end{displaymath}
}

\begin{proof}
Given the sampling distribution of the Unigram sampler $$ Q_{\text{unigram}} (i | \bm{z}) = q_i $$
the KL-divergence has the following bound:
\begin{displaymath}
    \begin{split}
        \mathcal{D}_{KL}\left[   Q_{\text{unigram}} (\cdot | \bm{z}) \Vert P  (\cdot | \bm{z}) \right] & = \sum_{i=1}^N Q_{\text{unigram}} (i | \bm{z}) \cdot \log \frac{Q_{\text{unigram}} (i | \bm{z})}{P (i | \bm{z})} \\
        & = \sum_{i=1}^N Q_{\text{unigram}} (i | \bm{z}) \cdot \log \frac{q_i \sum_{j=1}^N \exp (o_j) }{\exp (o_i) } \\
        & \le \sum_{i=1}^N Q_{\text{unigram}} (i | \bm{z}) \cdot \log \frac{q_{max} \sum_{j=1}^N \exp (\Vert \bm{o} \Vert_{\infty})}{\exp (- \Vert \bm{o} \Vert_{\infty})} \\ 
        & = \sum_{i=1}^N Q_{\text{unigram}} (i | \bm{z}) \cdot \log \exp ( 2 \Vert \bm{o} \Vert_{\infty} + \ln N q_{max}) \\
        & = 2 \Vert \bm{o} \Vert_{\infty} + \ln N q_{max}
    \end{split}
\end{displaymath}
where $q_{max}$ denotes the maximum frequency given the data set.
Considering the non-negativity of the KL-divergence, the lower bound of the KL-divergence should be greater than zero values.
\end{proof}

\noindent
{\bf Theorem~\ref{theorem:kl_midx_uni}} {\it
Assuming that the $\tilde{o}_i = \bm{z}^\top \tilde{\bm{q}}_i$ is the similarity score with respect to the residual embedding $ \bm{\tilde{q}}_i$ and the vector $\tilde{\bm{o}}\in \mathbb{R}^N$ denotes score vector over all classes, the KL divergence from the softmax distribution $P(\cdot|\bm{z})$ to the proposal distribution $Q_{\text{midx}} (\cdot| \bm{z})$ according to the midx sampler can be bounded from below:
	\begin{displaymath}
		0 \le \mathcal{D}_{KL}\left[   Q_{\text{midx}}(\cdot|\bm{z}) \Vert P(\cdot|\bm{z}) \right] \le 2  \Vert \tilde{\bm{o}} \Vert_\infty.
	\end{displaymath}
}

\begin{proof}
Given the sampling distribution of the \texttt{MIDX} sampler 
$$
Q_{\text{midx}} (i | \bm{z})  = \frac{\exp (o_i - \tilde{o}_i)}{\sum_{j=1}^N \exp (o_j - \tilde{o}_j)}
$$
KL divergence has the following bound:
\begin{displaymath}
    \begin{split}
        \mathcal{D}_{KL}\left[   Q_{\text{midx}}(\cdot|\bm{z}) \Vert P(\cdot|\bm{z}) \right] & = \sum_{j=1}^N Q_{\text{midx}}(i | \bm{z}) \cdot \log \frac{Q_{\text{midx}}(i | \bm{z})}{P(i|\bm{z})} \\
        & = \sum_{j=1}^N Q_{\text{midx}}(i | \bm{z}) \cdot \log \frac{\exp (o_i - \tilde{o}_i)}{\sum_{j=1}^N \exp (o_j - \tilde{o}_j)} \cdot \frac{\sum_{k=1}^N \exp (o_k)}{\exp (o_i)}  \\ 
        & = \sum_{j=1}^N Q_{\text{midx}}(i | \bm{z}) \cdot \log \frac{\sum_{k=1}^N \exp (o_k - \tilde{o}_k) \exp (\tilde{o}_k)}{\sum_{j=1}^N \exp (o_j - \tilde{o}_j) \exp (\tilde{o}_i)} \\ 
        & \le \sum_{j=1}^N Q_{\text{midx}}(i | \bm{z}) \cdot \log \frac{\sum_{k=1}^N \exp (o_k - \tilde{o}_k) }{\sum_{j=1}^N \exp (o_j - \tilde{o}_j) } \cdot \frac{\exp (\Vert \tilde{\bm{o}} \Vert_\infty)}{\exp (- \Vert \tilde{\bm{o}} \Vert_\infty)} \\ 
        & = \le \sum_{j=1}^N Q_{\text{midx}}(i | \bm{z}) \cdot \log \exp (2 \Vert \tilde{\bm{o}} \Vert_\infty) = 2 \Vert \tilde{\bm{o}} \Vert_\infty
    \end{split}
\end{displaymath}
$\mathcal{D}_{KL}\left[   Q_{\text{midx}}(\cdot|\bm{z}) \Vert P(\cdot|\bm{z}) \right] \ge 0 $ holds due to the non-negativity of the KL-divergence. 
\end{proof}

\subsection*{Appendix B.2 Proof in Section~\ref{sec:theo_gradient_bias}}

In this appendix, we prove the following theorem from
Section~\ref{sec:theo_gradient_bias}:

\noindent
{\bf Assumption~\ref{theorem:assump_gradient_bias}}
{\it The following conditions hold for the query encoder and classes encoder:

1. The encoder functions (mapping functions $\phi_c$ and $\phi_q$) are L-Lipschiz in the parameter $\theta$. In particular, we have $\Vert \phi\left( \theta\right) - \phi \left( \theta'\right) \Vert_2 \leq L \Vert \theta - \theta' \Vert_2$ for $\theta$, $\theta'$.

2. The similarity score functions have bounded gradients, i.e., we have $\Vert \nabla o_{j} \Vert_\infty \leq U$ for all $j \in [N]$.
}

\noindent
{\bf Theorem~\ref{theorem:gradient_error_general} }
{\it Let $s_1, s_2, ..., s_M$ i.i.d. random variables sampled from proposal $Q$, the gradient bias follows (\textbf{Proposition 7 in~\citep{metelli2020importance})}:} 
\begin{displaymath}
\begin{split}
    \vert \mathbb{E}[\nabla_{\theta_t} \ell'] - \nabla_{\theta_t} \ell \vert 
    & \le U \min \left\{2, \sqrt{\frac{d_2(P \Vert Q) - 1}{M + 1}}\right\}
\end{split}
\end{displaymath}
where $d_2(P\Vert Q) = \mathbb{E}_{i\sim P} [p_i/ q_i]$ denotes the exponential in base 2 of the Renyi divergences with the order of 2, which evaluates the difference from the target distribution. 
\begin{proof}
Recall that the \texttt{Full} softmax loss follows as:
    \begin{displaymath}
        \ell (\bm{y}, \bm{p}) = - \log \frac{\exp o_i}{ \sum_{j=1}^N\exp o_j}
    \end{displaymath}
    The gradient for the softmax loss follows:
    \begin{displaymath}
        \nabla_{\theta} \ell(\bm{y}, \bm{p}) = - \nabla_{\theta} o_i + \sum_{j=1}^N p_j \cdot \nabla_{\theta} o_j = - \nabla_{\theta} o_i + \mathbb{E}_{j \sim P} [\nabla_{\theta} o_j]
    \end{displaymath}
    where $p_j = \frac{\exp o_j }{\sum_{k=1}^N \exp o_k}$ denotes the softmax probability of class $j$. To efficiently estimate the gradient of the full softmax loss, the technique of importance sampling is used here to estimate the gradient. With the proposal distribution $Q$ including $M$ sampled classes and the one positive class, the importance weight is $w_i = p_i / q_i$. The estimator is:
    \begin{displaymath}
        \hat{u} = \frac{1}{M+1} \sum_{k=1}^{M+1} \frac{p_k}{q_k} \nabla_{\theta} o_k = \frac{1}{M+1} \sum_{k=1}^{M+1} w_k \nabla_{\theta} o_k
    \end{displaymath}
    where the samples $k$ are sampled from $Q$ independently. According to the properties in importance sampling, this estimator is unbiased, i.e., $\mathbb{E}_{i\sim Q}[\hat{\mu}] = \mathbb{E}_{i \sim P} [\nabla_{\theta} o_i]$. The self-normalized estimator is further adopted to mitigate the variance problem, which uses the normalized weight $\tilde{w}_i$:
    \begin{displaymath}
    \begin{split}
        \tilde{w}_i & = \frac{w_i}{\sum_{j=1}^{M+1} w_j} = \frac{p_i  / q_i }{ \sum_{j=1}^{M+1} p_j / q_j} \\ 
        & = \frac{\exp (o_i - \ln q_i)}{\sum_{k}^{N} \exp o_j} \cdot \frac{\sum_{l}^{N} \exp o_l}{ \sum_{j=1}^{M+1} \exp (o_j - \ln q_j) } \\
        & = \frac{\exp (o_i - \ln q_i)}{\sum_{j=1}^{M+1} \exp (o_j - \ln q_j)} = p'_i
    \end{split}
\end{displaymath}
where $p'_i$ denotes the softmax probability according to the corrected logits, i.e., $p'_i = \frac{\exp o'_i}{\sum_{j=1}^N \exp o'_j}$. The corrected logit $o'_j$ refers to the definition in Eq~\eqref{eq:correct_logit}. The corresponding estimator follows as:
\begin{displaymath}
    \begin{split}
        \tilde{\mu} = \sum_{k=1}^{M+1} \tilde{w}_k \nabla_{\theta_t} o_k = \sum_{k=1}^{M+1} p'_k \nabla_{\theta_t} o_k 
    \end{split}
\end{displaymath}
Let's review the sampled softmax and its gradient.
\begin{displaymath}
    \ell'(\bm{y}', \bm{p}') = - \log \frac{ \exp o'_i}{\sum_{j=1}^{M+1} \exp o'_j}, \nabla_{\theta} \ell'(\bm{y}', \bm{p}') = - \nabla_{\bm{\theta}} o_i + \sum_{j=1}^{M+1} p'_j \cdot \nabla_{\bm{\theta}} o_j = - \nabla_{\bm{\theta}} o_i + \tilde{\mu}
\end{displaymath}
where $M$ denotes the number of sampled negative classes. 
Different from $\hat{\mu}$, $\tilde{\mu}$ is biased but consistent. Obviously, $\vert \tilde{\mu}\vert \le \Vert \nabla_{\theta_t} o_j \Vert\infty$ holds. 

To quantify the discrepancy between the proposal distribution $Q$ and the target distribution $P$, here we will use the Renyi divergence with second-order and its exponential. 
\begin{displaymath}
    D_2 \left( P \Vert Q \right) = \log \sum_{i=1}^N p_i \cdot \frac{p_i}{q_i}, \quad d_2 \left( P \Vert Q \right) = \sum_{i=1}^N p_i \cdot \frac{p_i}{q_i}
\end{displaymath}
The gradient bias for the sampled softmax with the proposal distribution follows:
\begin{align}
    \left\vert \mathbb{E}_{i\sim Q}[\nabla_{\theta} \ell' ] - \nabla_{\theta}\ell \right \vert 
    & = \left\vert \mathbb{E}_{i\sim Q}\left[\tilde{\mu} \right] -  \mathbb{E}_{j\sim P}\left[\nabla_{\theta} o_j\right]\right \vert \nonumber\\ 
    & = \left\vert \mathbb{E}_{i \sim Q} \left[ \tilde{\mu} - \mathbb{E}_{j\sim P}\left[\nabla_{\theta} o_j\right]\right] \right\vert \nonumber\\
    & = \left\vert \mathbb{E}_{i\sim Q} \left[ \tilde{\mu} - \mathbb{E}_{i\sim Q} [\hat{\mu}]\right] \right\vert = \left\vert \mathbb{E}_{i\sim Q} \left[ \tilde{\mu} - \hat{\mu}\right] \right\vert \nonumber \\ 
    & \le \mathbb{E}_{i \sim Q} \left[ \left\vert \tilde{\mu} - \hat{\mu} \right\vert \right] \nonumber \\ 
    & = \mathbb{E}_{i \sim Q} \left[ \left\vert \frac{\sum_{i=1}^{M+1} w_i \nabla_{\theta_t} o_i}{\sum_{j=1}^{M+1} w_j} - \frac{\sum_{i=1}^{M+1} w_i \nabla_{\theta_t} o_i}{M+1}\right\vert \right] \nonumber\\ 
    & = \mathbb{E}_{i \sim Q} \left[ \left\vert \frac{\sum_{i=1}^{M+1} w_i \nabla_{\theta_t} o_i}{\sum_{j=1}^{M+1} w_j}\right\vert \left\vert 1 - \frac{\sum_{i=1}^{M+1} w_i}{M+1}\right\vert \right] \nonumber\\ 
    & \overset{(A)}{\le} \mathbb{E}_{i \sim Q} \left[ \left( \frac{\sum_{i=1}^{M+1} w_i \nabla_{\theta_t} o_i}{\sum_{j=1}^{M+1} w_j}\right)^2 \right]^{\frac{1}{2}} \cdot \mathbb{E}_{i \sim Q} \left[ \left( 1 - \frac{\sum_{i=1}^{M+1} w_i}{M+1} \right)^2 \right]^{\frac{1}{2}}  \nonumber\\ 
    & \overset{(B)}{\le} \Vert \nabla_{\theta_t} o_j \Vert\infty  \cdot \mathbb{E}_{i \sim Q} \left[ \left( 1 - \frac{\sum_{i=1}^{M+1} w_i}{M+1} \right)^2 \right]^{\frac{1}{2}}\nonumber\\
    & \overset{(C)}{\le} U \sqrt{\frac{d_2(P \Vert Q) - 1}{M + 1}} \label{eq:gradient_bias_overall}
\end{align}
Step $(A)$ follows by applying Cauchy–Schwartz inequality. Step $(B)$ follows depending on the observation that $\left( \tilde{w}_i \nabla_{\theta_t} o_j \right)^2 \le \Vert \nabla_{\theta_t} o_j \Vert^2_\infty $.  $(C)$ holds following Lemma 1 in~\citep{cortes2010learning}.
\end{proof}

According to the above theorem, the gradient bias is bounded by the distribution divergence and the number of samples. With a greater divergence between the softmax distribution and the proposal distribution, the upper bound of the gradient bias would get greater. With more classes sampled, the estimator will be less-biased. For the different proposals, we will derive the divergence from the softmax distribution and get the corresponding results. 

\noindent
{\bf Theorem~\ref{theorem:gradient_error_uni}}
{\it Let $s_1, s_2, ..., s_M$ i.i.d. random variables sampled from uniform proposal $Q$, the gradient approximation is bounded by:}
\begin{displaymath}
    \left\vert \mathbb{E}[\nabla_{ \theta_t} \ell'] - \nabla_{ \theta_t}\ell \right\vert  \le \min \left\{2, U \sqrt{\frac{\exp \left(2 \Vert \bm{o} \Vert_\infty \right) - 1 }{M + 1}} \right\}
\end{displaymath}
\begin{proof}
With the uniform sampler, the divergence follows:
\begin{displaymath}
\begin{split}
    d_2(P\Vert Q_{uni}) & = \mathbb{E}_{i\sim P} \left[ \frac{p_i}{q_i}\right] = \mathbb{E}_{i\sim P}\left[  \frac{N \exp o_i}{\sum_{j=1}^N \exp o_j} \right] \\ 
    & \le \sum_{i=1}^N \frac{\exp o_i}{\sum_{j=1}^N \exp o_j}  \exp \left(2 \Vert \bm{o} \Vert_\infty \right) = \exp \left(2 \Vert \bm{o} \Vert_\infty \right)
\end{split}
\end{displaymath}
Putting the results in Theorem~\ref{theorem:gradient_error_general}, we can get the approximation for the uniform sampler:
\begin{displaymath}
    \left\vert \mathbb{E}[\nabla_{ \theta_t} \ell'] - \nabla_{ \theta_t}\ell \right\vert  \le \min \left\{2, U \sqrt{\frac{\exp \left(2 \Vert \bm{o} \Vert_\infty \right) - 1 }{M + 1}} \right\}
\end{displaymath}
\end{proof}

\noindent
{\bf Theorem~\ref{theorem:gradient_error_pop}}
{\it Let $s_1, s_2, ..., s_M$ i.i.d. random variables sampled from unigram proposal $Q$, the gradient approximation is bounded by:}
\begin{displaymath}
    \left\vert \mathbb{E}[\nabla_{ \theta_t} \ell'] - \nabla_{ \theta_t}\ell \right\vert \le \min \left\{2, U \sqrt{\frac{\exp \left(2 \Vert \bm{o} \Vert_\infty - \ln q_{min}\right) - 1 }{M + 1}} \right\} 
\end{displaymath}
\begin{proof}
With the unigram sampler, the divergence follows:
\begin{displaymath}
\begin{split}
    d_2(P\Vert Q_{unigram}) & = \mathbb{E}_{i\sim P} \left[ \frac{p_i}{q_i}\right] = \mathbb{E}_{i\sim P}\left[   \frac{\exp o_i}{q_i \sum_{j=1}^N \exp o_j} \right] \\ 
    & \le \sum_{i=1}^N \frac{\exp o_i}{\sum_{j=1}^N \exp o_j}  \exp \left(2 \Vert \bm{o} \Vert_\infty - \ln q_{min} \right) = \exp \left(2 \Vert \bm{o} \Vert_\infty - \ln q_{min}\right)
\end{split}
\end{displaymath}
where $q_{min}$ denotes the minimum unigram probability over the class.
Putting the results in Theorem~\ref{theorem:gradient_error_general}, we can get the approximation for the unigram sampler:
\begin{displaymath}
    \left\vert \mathbb{E}[\nabla_{ \theta_t} \ell'] - \nabla_{ \theta_t}\ell \right\vert  \le \min \left\{2, U \sqrt{\frac{\exp \left(2 \Vert \bm{o} \Vert_\infty - \ln q_{min}\right) - 1 }{M + 1}} \right\}
\end{displaymath}
\end{proof}

\noindent
{\bf Theorem~\ref{theorem:gradient_error_midx}}
{\it Let $s_1, s_2, ..., s_M$ i.i.d. random variables sampled from \texttt{MIDX} proposal $Q$, the gradient approximation is bounded by:}
\begin{displaymath}
    \left\vert \mathbb{E}[\nabla_{ \theta_t} \ell'] - \nabla_{ \theta_t}\ell \right\vert \le \min \left\{2, U \sqrt{\frac{\exp (2 \Vert \tilde{\bm{o}} \Vert_\infty) - 1 }{M + 1}} \right\} 
\end{displaymath}
\begin{proof}
With the \texttt{MIDX} sampler, the divergence follows:
\begin{displaymath}
\begin{split}
    d_2(P\Vert Q_{midx}) & = \mathbb{E}_{i\sim P} \left[ \frac{p_i}{q_i}\right] = \mathbb{E}_{i\sim P}\left[   \frac{\exp o_i}{ \sum_{j=1}^N \exp o_j}  \frac{\sum_{k-1}^N \exp (o_k - \tilde{o}_k)}{\exp (o_i - \tilde{o_i})}\right] \\ 
    & = \mathbb{E}_{i\sim P}\left[   \frac{\exp \tilde{o}_i  \sum_{k-1}^N \exp (o_k - \tilde{o}_k)}{ \sum_{j=1}^N \exp (o_j - \tilde{o}_j) \exp \tilde{o}_j } \right] \\ 
    & \le \mathbb{E}_{i\sim P}\left[ \exp (2 \Vert \tilde{\bm{o}} \Vert_\infty) \frac{ \sum_{k-1}^N \exp (o_k - \tilde{o}_k)}{ \sum_{j=1}^N \exp (o_j - \tilde{o}_j)  } \right] = \exp (2 \Vert \tilde{\bm{o}} \Vert_\infty)
\end{split}
\end{displaymath}
Putting the results in Theorem~\ref{theorem:gradient_error_general}, we can get the approximation for the \texttt{MIDX} sampler:
\begin{displaymath}
    \left\vert \mathbb{E}[\nabla_{ \theta_t} \ell'] - \nabla_{ \theta_t}\ell \right\vert  \le \min \left\{2, U \sqrt{\frac{\exp (2 \Vert \tilde{\bm{o}} \Vert_\infty) - 1 }{M + 1}} \right\}
\end{displaymath}
\end{proof}

\vskip 0.2in

\bibliography{sample}

\begin{thebibliography}{44}
\providecommand{\natexlab}[1]{#1}
\providecommand{\url}[1]{\texttt{#1}}
\expandafter\ifx\csname urlstyle\endcsname\relax
  \providecommand{\doi}[1]{doi: #1}\else
  \providecommand{\doi}{doi: \begingroup \urlstyle{rm}\Url}\fi

\bibitem[Aly(2005)]{aly2005survey}
M.~Aly.
\newblock Survey on multiclass classification methods.
\newblock \emph{Neural Netw}, 19\penalty0 (1-9):\penalty0 2, 2005.

\bibitem[Andoni and Indyk(2008)]{andoni2008near}
A.~Andoni and P.~Indyk.
\newblock Near-optimal hashing algorithms for approximate nearest neighbor in
  high dimensions.
\newblock \emph{Communications of the ACM}, 51\penalty0 (1):\penalty0 117--122,
  2008.

\bibitem[Babenko and Lempitsky(2014)]{babenko2014inverted}
A.~Babenko and V.~Lempitsky.
\newblock The inverted multi-index.
\newblock \emph{IEEE transactions on pattern analysis and machine
  intelligence}, 37\penalty0 (6):\penalty0 1247--1260, 2014.

\bibitem[Bengio and Sen{\'e}cal(2008)]{bengio2008adaptive}
Y.~Bengio and J.-S. Sen{\'e}cal.
\newblock Adaptive importance sampling to accelerate training of a neural
  probabilistic language model.
\newblock \emph{IEEE Transactions on Neural Networks}, 19\penalty0
  (4):\penalty0 713--722, 2008.

\bibitem[Bentley(1975)]{bentley1975multidimensional}
J.~L. Bentley.
\newblock Multidimensional binary search trees used for associative searching.
\newblock \emph{Communications of the ACM}, 18\penalty0 (9):\penalty0 509--517,
  1975.

\bibitem[Bishop and Nasrabadi(2006)]{bishop2006pattern}
C.~M. Bishop and N.~M. Nasrabadi.
\newblock \emph{Pattern recognition and machine learning}, volume~4.
\newblock Springer, 2006.

\bibitem[Blanc and Rendle(2018)]{blanc2018adaptive}
G.~Blanc and S.~Rendle.
\newblock Adaptive sampled softmax with kernel based sampling.
\newblock In \emph{International conference on machine learning}, pages
  590--599. PMLR, 2018.

\bibitem[Cortes et~al.(2010)Cortes, Mansour, and Mohri]{cortes2010learning}
C.~Cortes, Y.~Mansour, and M.~Mohri.
\newblock Learning bounds for importance weighting.
\newblock \emph{Advances in neural information processing systems}, 23, 2010.

\bibitem[Douze et~al.(2024)Douze, Guzhva, Deng, Johnson, Szilvasy, Mazar{\'e},
  Lomeli, Hosseini, and J{\'e}gou]{douze2024faiss}
M.~Douze, A.~Guzhva, C.~Deng, J.~Johnson, G.~Szilvasy, P.-E. Mazar{\'e},
  M.~Lomeli, L.~Hosseini, and H.~J{\'e}gou.
\newblock The faiss library.
\newblock \emph{arXiv preprint arXiv:2401.08281}, 2024.

\bibitem[Fu et~al.()Fu, Xiang, Wang, and Cai]{fu12fast}
C.~Fu, C.~Xiang, C.~Wang, and D.~Cai.
\newblock Fast approximate nearest neighbor search with the navigating
  spreading-out graph.
\newblock \emph{Proceedings of the VLDB Endowment}, 12\penalty0 (5).

\bibitem[Ge et~al.(2013)Ge, He, Ke, and Sun]{ge2013optimized}
T.~Ge, K.~He, Q.~Ke, and J.~Sun.
\newblock Optimized product quantization.
\newblock \emph{IEEE transactions on pattern analysis and machine
  intelligence}, 36\penalty0 (4):\penalty0 744--755, 2013.

\bibitem[Gersho and Gray(2012)]{gersho2012vector}
A.~Gersho and R.~M. Gray.
\newblock \emph{Vector quantization and signal compression}, volume 159.
\newblock Springer Science \& Business Media, 2012.

\bibitem[Goodman(2001)]{goodman2001classes}
J.~Goodman.
\newblock Classes for fast maximum entropy training.
\newblock In \emph{2001 IEEE International Conference on Acoustics, Speech, and
  Signal Processing. Proceedings (Cat. No. 01CH37221)}, volume~1, pages
  561--564. IEEE, 2001.

\bibitem[Gutmann and Hyv{\"a}rinen(2010)]{gutmann2010noise}
M.~Gutmann and A.~Hyv{\"a}rinen.
\newblock Noise-contrastive estimation: A new estimation principle for
  unnormalized statistical models.
\newblock In \emph{Proceedings of the thirteenth international conference on
  artificial intelligence and statistics}, pages 297--304. JMLR Workshop and
  Conference Proceedings, 2010.

\bibitem[Guttman(1984)]{guttman1984r}
A.~Guttman.
\newblock R-trees: A dynamic index structure for spatial searching.
\newblock In \emph{Proceedings of the 1984 ACM SIGMOD international conference
  on Management of data}, pages 47--57, 1984.

\bibitem[Jang et~al.(2016)Jang, Gu, and Poole]{jang2016categorical}
E.~Jang, S.~Gu, and B.~Poole.
\newblock Categorical reparameterization with gumbel-softmax.
\newblock \emph{arXiv preprint arXiv:1611.01144}, 2016.

\bibitem[Jannach and Ludewig(2017)]{jannach2017recurrent}
D.~Jannach and M.~Ludewig.
\newblock When recurrent neural networks meet the neighborhood for
  session-based recommendation.
\newblock In \emph{Proceedings of the eleventh ACM conference on recommender
  systems}, pages 306--310, 2017.

\bibitem[Jegou et~al.(2010)Jegou, Douze, and Schmid]{jegou2010product}
H.~Jegou, M.~Douze, and C.~Schmid.
\newblock Product quantization for nearest neighbor search.
\newblock \emph{IEEE transactions on pattern analysis and machine
  intelligence}, 33\penalty0 (1):\penalty0 117--128, 2010.

\bibitem[Kang and McAuley(2018)]{kang2018self}
W.-C. Kang and J.~McAuley.
\newblock Self-attentive sequential recommendation.
\newblock In \emph{2018 IEEE international conference on data mining (ICDM)},
  pages 197--206. IEEE, 2018.

\bibitem[Kulis and Darrell(2009)]{kulis2009learning}
B.~Kulis and T.~Darrell.
\newblock Learning to hash with binary reconstructive embeddings.
\newblock \emph{Advances in neural information processing systems}, 22, 2009.

\bibitem[Li et~al.(2019)Li, Zhang, Sun, Wang, Li, Zhang, and
  Lin]{li2019approximate}
W.~Li, Y.~Zhang, Y.~Sun, W.~Wang, M.~Li, W.~Zhang, and X.~Lin.
\newblock Approximate nearest neighbor search on high dimensional
  data—experiments, analyses, and improvement.
\newblock \emph{IEEE Transactions on Knowledge and Data Engineering},
  32\penalty0 (8):\penalty0 1475--1488, 2019.

\bibitem[Lindgren et~al.(2021)Lindgren, Reddi, Guo, and
  Kumar]{lindgren2021efficient}
E.~Lindgren, S.~Reddi, R.~Guo, and S.~Kumar.
\newblock Efficient training of retrieval models using negative cache.
\newblock \emph{Advances in Neural Information Processing Systems},
  34:\penalty0 4134--4146, 2021.

\bibitem[Malkov and Yashunin(2018)]{malkov2018efficient}
Y.~A. Malkov and D.~A. Yashunin.
\newblock Efficient and robust approximate nearest neighbor search using
  hierarchical navigable small world graphs.
\newblock \emph{IEEE transactions on pattern analysis and machine
  intelligence}, 42\penalty0 (4):\penalty0 824--836, 2018.

\bibitem[Marcus et~al.(1993)Marcus, Santorini, and
  Marcinkiewicz]{marcus-etal-1993-building}
M.~P. Marcus, B.~Santorini, and M.~A. Marcinkiewicz.
\newblock Building a large annotated corpus of {E}nglish: The {P}enn
  {T}reebank.
\newblock \emph{Computational Linguistics}, 19\penalty0 (2):\penalty0 313--330,
  1993.
\newblock URL \url{https://aclanthology.org/J93-2004}.

\bibitem[Martins and Astudillo(2016)]{martins2016softmax}
A.~Martins and R.~Astudillo.
\newblock From softmax to sparsemax: A sparse model of attention and
  multi-label classification.
\newblock In \emph{International conference on machine learning}, pages
  1614--1623. PMLR, 2016.

\bibitem[Menon et~al.(2019)Menon, Rawat, Reddi, and Kumar]{menon2019multilabel}
A.~K. Menon, A.~S. Rawat, S.~Reddi, and S.~Kumar.
\newblock Multilabel reductions: what is my loss optimising?
\newblock \emph{Advances in Neural Information Processing Systems}, 32, 2019.

\bibitem[Merity et~al.(2016)Merity, Xiong, Bradbury, and
  Socher]{merity2016pointer}
S.~Merity, C.~Xiong, J.~Bradbury, and R.~Socher.
\newblock Pointer sentinel mixture models.
\newblock \emph{arXiv preprint arXiv:1609.07843}, 2016.

\bibitem[Metelli et~al.(2020)Metelli, Papini, Montali, and
  Restelli]{metelli2020importance}
A.~M. Metelli, M.~Papini, N.~Montali, and M.~Restelli.
\newblock Importance sampling techniques for policy optimization.
\newblock \emph{Journal of Machine Learning Research}, 21\penalty0
  (141):\penalty0 1--75, 2020.

\bibitem[Mikolov et~al.(2013{\natexlab{a}})Mikolov, Chen, Corrado, and
  Dean]{mikolov2013efficient}
T.~Mikolov, K.~Chen, G.~Corrado, and J.~Dean.
\newblock Efficient estimation of word representations in vector space.
\newblock \emph{arXiv preprint arXiv:1301.3781}, 2013{\natexlab{a}}.

\bibitem[Mikolov et~al.(2013{\natexlab{b}})Mikolov, Sutskever, Chen, Corrado,
  and Dean]{mikolov2013distributed}
T.~Mikolov, I.~Sutskever, K.~Chen, G.~S. Corrado, and J.~Dean.
\newblock Distributed representations of words and phrases and their
  compositionality.
\newblock \emph{Advances in neural information processing systems}, 26,
  2013{\natexlab{b}}.

\bibitem[Mnih and Hinton(2008)]{mnih2008scalable}
A.~Mnih and G.~E. Hinton.
\newblock A scalable hierarchical distributed language model.
\newblock \emph{Advances in neural information processing systems}, 21, 2008.

\bibitem[Mnih and Kavukcuoglu(2013)]{mnih2013learning}
A.~Mnih and K.~Kavukcuoglu.
\newblock Learning word embeddings efficiently with noise-contrastive
  estimation.
\newblock \emph{Advances in neural information processing systems}, 26, 2013.

\bibitem[Morin and Bengio(2005)]{morin2005hierarchical}
F.~Morin and Y.~Bengio.
\newblock Hierarchical probabilistic neural network language model.
\newblock In \emph{International workshop on artificial intelligence and
  statistics}, pages 246--252. PMLR, 2005.

\bibitem[Mussmann and Ermon(2016)]{mussmann2016learning}
S.~Mussmann and S.~Ermon.
\newblock Learning and inference via maximum inner product search.
\newblock In \emph{International Conference on Machine Learning}, pages
  2587--2596. PMLR, 2016.

\bibitem[Rawat et~al.(2019)Rawat, Chen, Yu, Suresh, and
  Kumar]{rawat2019sampled}
A.~S. Rawat, J.~Chen, F.~X.~X. Yu, A.~T. Suresh, and S.~Kumar.
\newblock Sampled softmax with random fourier features.
\newblock \emph{Advances in Neural Information Processing Systems}, 32, 2019.

\bibitem[Robert(1999)]{robert1999monte}
C.~Robert.
\newblock Monte carlo statistical methods, 1999.

\bibitem[Salakhutdinov(2015)]{salakhutdinov2015learning}
R.~Salakhutdinov.
\newblock Learning deep generative models.
\newblock \emph{Annual Review of Statistics and Its Application}, 2\penalty0
  (1):\penalty0 361--385, 2015.

\bibitem[Shrivastava and Li(2014)]{shrivastava2014asymmetric}
A.~Shrivastava and P.~Li.
\newblock Asymmetric lsh (alsh) for sublinear time maximum inner product search
  (mips).
\newblock \emph{Advances in neural information processing systems}, 27, 2014.

\bibitem[Spring and Shrivastava(2017)]{spring2017new}
R.~Spring and A.~Shrivastava.
\newblock A new unbiased and efficient class of lsh-based samplers and
  estimators for partition function computation in log-linear models.
\newblock \emph{arXiv preprint arXiv:1703.05160}, 2017.

\bibitem[Vijayanarasimhan et~al.(2014)Vijayanarasimhan, Shlens, Monga, and
  Yagnik]{vijayanarasimhan2014deep}
S.~Vijayanarasimhan, J.~Shlens, R.~Monga, and J.~Yagnik.
\newblock Deep networks with large output spaces.
\newblock \emph{arXiv preprint arXiv:1412.7479}, 2014.

\bibitem[Walker(1977)]{walker1977efficient}
A.~J. Walker.
\newblock An efficient method for generating discrete random variables with
  general distributions.
\newblock \emph{ACM Transactions on Mathematical Software (TOMS)}, 3\penalty0
  (3):\penalty0 253--256, 1977.

\bibitem[Yang et~al.(2024)Yang, Ding, Huang, Cen, Song, Xu, Dong, and
  Tang]{yang2024does}
Z.~Yang, M.~Ding, T.~Huang, Y.~Cen, J.~Song, B.~Xu, Y.~Dong, and J.~Tang.
\newblock Does negative sampling matter? a review with insights into its theory
  and applications.
\newblock \emph{IEEE Transactions on Pattern Analysis and Machine
  Intelligence}, 2024.

\bibitem[Yi et~al.(2019)Yi, Yang, Hong, Cheng, Heldt, Kumthekar, Zhao, Wei, and
  Chi]{yi2019sampling}
X.~Yi, J.~Yang, L.~Hong, D.~Z. Cheng, L.~Heldt, A.~Kumthekar, Z.~Zhao, L.~Wei,
  and E.~Chi.
\newblock Sampling-bias-corrected neural modeling for large corpus item
  recommendations.
\newblock In \emph{Proceedings of the 13th ACM conference on recommender
  systems}, pages 269--277, 2019.

\bibitem[Zhang et~al.(2013)Zhang, Chen, Wang, and Yu]{zhang2013optimizing}
W.~Zhang, T.~Chen, J.~Wang, and Y.~Yu.
\newblock Optimizing top-n collaborative filtering via dynamic negative item
  sampling.
\newblock In \emph{Proceedings of the 36th international ACM SIGIR conference
  on Research and development in information retrieval}, pages 785--788, 2013.

\end{thebibliography}

\end{document}